\documentclass{article}

\usepackage[margin=1.125in]{geometry}
\usepackage{xcolor}
\definecolor{DarkGreen}{rgb}{0.1,0.5,0.1}
\definecolor{DarkRed}{rgb}{0.5,0.1,0.1}
\definecolor{DarkBlue}{rgb}{0.1,0.1,0.5}

\usepackage[small]{caption}
\usepackage{amssymb,amsmath,amsthm,amsfonts,enumitem}
\usepackage{fullpage,nicefrac,comment}
\usepackage{tikz}
\usetikzlibrary{arrows,decorations.pathmorphing,decorations.shapes,snakes,patterns,positioning}
\usepackage[ruled,linesnumbered]{algorithm2e}

\newcommand{\SI}{{\sc SubsIt}\xspace}
\newcommand{\NSI}{{\sc NSI}\xspace}
\newcommand{\AltLS}{{\sc S-M-AltLS}\xspace}

\newcommand{\Deflate}{{\sc SoftDeflate}\xspace}
\newcommand{\GS}{{\sc QR}\xspace}

\newcommand{\SmoothGS}{{\sc SmoothQR}\xspace}

\renewcommand{\tilde}{\widetilde}
\newcommand{\mustar}{\mu^*}
\newcommand{\widebar}[1]{\overline{#1}}
\renewcommand{\bar}{\widebar}
\newcommand{\trans}{T}

\newcommand{\Truncate}{\mathcal{T}}
\newcommand{\hyps}{the inductive hypotheses (\ref{h1}), (\ref{h2}), (\ref{h3}), and (\ref{eq:s}) }

\newcommand{\Yprev}{Y_{t-1}}
\newcommand{\Yt}{Y_t}
\newcommand{\Xprev}{X_{t-1}}
\newcommand{\Xt}{X_t}
\newcommand{\Mprev}{M^{(<t)}}
\newcommand{\Mt}{M^{(\leq t)}}
\newcommand{\Nprev}{N_{t-1}}
\newcommand{\Nt}{ N_t }
\newcommand{\Vt}{ V_t }

\newcommand{\NprevEst}{\widetilde{ \Nprev} }

\newcommand{\xlj}[1]{{#1}^{(\leq j)}}
\newcommand{\xlt}[1]{{#1}^{(\leq t)}}

\newcommand{\tY}{Y^T}

\newcommand{\PiX}{\Pi_X}
\newcommand{\PiXPerp}{\Pi_{X_\perp}}

\newcommand{\Smax}{\ensuremath{s_{\max}}}

\newcommand{\Cerr}{C_0}
\newcommand{\Cgapconstant}{C_1}
\newcommand{\Cmu}{C_2 } 

\newcommand{\Cinc}{C_5}
\newcommand{\Cnew}{C_3}

\newcommand{\cU}{\ensuremath{\mathcal{U}}}
\newcommand{\cD}{\ensuremath{\mathcal{D}}}

\newcommand{\R}{{\mathbb R}}

\newcommand{\PR}[1]{{\mathbb{P}}\left\{ #1\right\}}

\newcommand{\EE}{\mathbb{E}}

\newcommand{\norm}[1]{\left\|#1\right\|}
\newcommand{\twonorm}[1]{\left\|#1\right\|_2}

\newcommand{\infnorm}[1]{\left\|#1\right\|_\infty}
\newcommand{\opnorm}[1]{\left\|#1\right\|}

\newcommand{\decnorm}[2]{\left\|#1\right\|_{#2}}
\newcommand{\fronorm}[1]{\decnorm{#1}{F}}

\newcommand{\inabs}[1]{\left|#1\right|}
\newcommand{\ip}[2]{\ensuremath{\left\langle #1,#2\right\rangle}}

\newcommand{\inset}[1]{\left\{#1\right\}}
\newcommand{\inbrac}[1]{\left\{#1\right\}}
\newcommand{\inparen}[1]{\left(#1\right)}
\newcommand{\inbrak}[1]{\left[#1\right]}
\newcommand{\suchthat}{\,:\,}

\newcommand{\argmin}{\mathrm{argmin}}

\newcommand{\poly}{\mathrm{poly}}
\newcommand{\median}{\ensuremath{\operatorname{median}}}

\newcommand{\diag}{\ensuremath{\operatorname{diag}}}

\newcommand{\ind}[1]{\ensuremath{\mathbf{1}_{#1}}}

\newcommand{\eps}{\varepsilon}
\renewcommand{\epsilon}{\varepsilon}

\newtheorem{theorem}{Theorem} 
\newtheorem{lemma}[theorem]{Lemma}

\newtheorem{remark}{Remark}
\newtheorem{claim}[theorem]{Claim}
\newtheorem{proposition}[theorem]{Proposition}

\def\draft{1} 

\def\submit{0} 

\ifnum\draft=1 
    \def\ShowAuthNotes{1}
\else
    \def\ShowAuthNotes{0}
\fi

\ifnum\submit=1
\newcommand{\forsubmit}[1]{#1}
\newcommand{\forreals}[1]{}
\else
\newcommand{\forreals}[1]{#1}
\newcommand{\forsubmit}[1]{}
\fi
\ifnum\ShowAuthNotes=1
\newcommand{\authnote}[2]{{ \footnotesize \bf{\color{DarkRed}[#1's Note:
{\color{DarkBlue}#2}]}}}
\else
\newcommand{\authnote}[2]{}
\fi

\usepackage[T1]{fontenc}
\usepackage{kpfonts}
\usepackage{microtype}

\newcommand{\chapterref}[1]{\hyperref[ch:#1]{Chapter~\ref{ch:#1}}}

\newcommand{\claimref}[1]{\hyperref[claim:#1]{Claim~\ref{claim:#1}}}

\newcommand{\corollaryref}[1]{\hyperref[cor:#1]{Corollary~\ref{cor:#1}}}

\newcommand{\definitionref}[1]{\hyperref[def:#1]{Definition~\ref{def:#1}}}
\newcommand{\equationlabel}[1]{\label{eq:#1}}
\newcommand{\equationref}[1]{\hyperref[eq:#1]{Equation~\ref{eq:#1}}}

\newcommand{\factref}[1]{\hyperref[fact:#1]{Fact~\ref{fact:#1}}}

\newcommand{\figureref}[1]{\hyperref[fig:#1]{Figure~\ref{fig:#1}}}

\newcommand{\itemref}[1]{\hyperref[item:#1]{Item~(\ref{item:#1})}}

\newcommand{\lemmaref}[1]{\hyperref[lem:#1]{Lemma~\ref{lem:#1}}}

\newcommand{\propref}[1]{\hyperref[prop:#1]{Proposition~\ref{prop:#1}}}

\newcommand{\propositionref}[1]{\hyperref[prop:#1]{Proposition~\ref{prop:#1}}}

\newcommand{\remarkref}[1]{\hyperref[rem:#1]{Remark~\ref{rem:#1}}}
\newcommand{\sectionlabel}[1]{\label{sec:#1}}
\newcommand{\sectionref}[1]{\hyperref[sec:#1]{Section~\ref{sec:#1}}}
\newcommand{\theoremlabel}[1]{\label{thm:#1}}
\newcommand{\theoremref}[1]{\hyperref[thm:#1]{Theorem~\ref{thm:#1}}}

\newcommand{\Set}[1]{\left\{#1\right\}}

\newcommand{\mper}{\,.}
\newcommand{\mcom}{\,,}

\title{Fast Matrix Completion Without the Condition Number}
\author{ Moritz Hardt \and Mary Wootters\thanks{MW's work was supported in part by the Simons Institute and by a Rackham predoctoral fellowship.} }
\date{\today}
\begin{document}
\maketitle
\begin{abstract} 
We give the first algorithm for Matrix Completion whose running time
and sample complexity is polynomial in the rank of the unknown target matrix,
\emph{linear} in the dimension of the matrix, and \emph{logarithmic} in the 
condition number of the matrix. 
To the best of our knowledge, all previous algorithms either incurred a
quadratic dependence on the condition number of the unknown matrix or a
quadratic dependence on the dimension of the matrix in the running time.

Our algorithm is based on a novel extension of Alternating Minimization which we
show has theoretical guarantees under standard assumptions even in the presence
of noise. 
\end{abstract}

\section{Introduction}\label{intro}
Matrix Completion is the problem of recovering an unknown real-valued low-rank
matrix from a possibly noisy subsample of its entries. The problem has received a
tremendous amount of attention in signal processing and machine learning
partly due to its wide applicability to recommender systems. A beautiful line
of work showed that a particular convex program---known as nuclear
norm minimization---achieves strong recovery guarantees under certain
reasonable feasibility assumptions~\cite{CandesR09,CandesT10,RechtFP10,Recht11}.
Nuclear norm minimization boils down to solving a semidefinite program and
therefore can be solved in polynomial time in the dimension of the matrix.
Unfortunately, the approach is not immediately practical due to the large
polynomial dependence on the dimension of the matrix. An ongoing research
effort aims to design large-scale algorithms for nuclear norm
minimization~\cite{JiY09,MazumderHT10,JaggiS10,AvronKKS12,HsiehO14}. 
Such fast solvers, generally speaking, involve heuristics that improve
empirical performance but may no longer preserve the strong theoretical
guarantees of the nuclear norm approach.

A successful scalable algorithmic alternative to Nuclear Norm Minimization is
based on Alternating Minimization~\cite{BellK07,HaldarH09,KorenBV09}. 
Alternating Minimization aims to recover the unknown low-rank
matrix by alternatingly optimizing over one of two factors in a purported low-rank
decomposition. Each update is a simple least squares regression problem that 
can be solved very efficiently. As pointed out in~\cite{HsiehO14}, even state
of the art nuclear norm solvers often cannot compete with Alternating
Minimization with regards to scalability. A shortcoming of Alternating Minimization is
that formal guarantees are less developed than for Nuclear Norm Minimization.
Only recently has there been progress in this
direction~\cite{Keshavan12,JainNS13,GunasekarAGG13,Hardt13b}. 

Unfortunately, 
despite this recent progress all known convergence
bounds for Alternating Minimization have at least a quadratic dependence on the
\emph{condition number} of the matrix. Here, the condition number refers to
the ratio of the first to the $k$-th singular value of the matrix, where $k$ is the
target rank of the decomposition. This dependence on the condition number can be a 
serious shortcoming. After all, Matrix Completion rests on the assumption that the
unknown matrix is approximately low-rank and hence we should expect its
singular values to decay rapidly. Indeed, strongly decaying singular values are a 
typical feature of large real-world matrices. 

The dependence on the condition number in Alternating Minimization is
not a mere artifact of the analysis. It arises naturally with the use of
the Singular Value Decomposition (SVD). Alternating Minimization is
typically intialized with a decomposition based on a truncated SVD of the
partial input matrix. Such an approach must incur a polynomial dependence on
the condition number. Many other approaches also crucially rely on the SVD as a
sub-routine, e.g., \cite{JainMD10,KeshavanMO10,KeshavanMO10b}, as well as most fast
solvers for the nuclear norm. 
In fact, there appears to be a kind of dichotomy in the current literature on
Matrix Completion: either the algorithm is \emph{not fast} and has at least a
quadratic dependence on the dimension of the matrix in its running time, or it
is \emph{not well-conditioned} and has at least a quadratic dependence on the
condition number in the sample complexity. We emphasize that here we focus on formal 
guarantees rather than observed empirical performance which may be better on certain
instances. This situation leads us to the following problem.
\begin{center}
{\bf Main Problem:} Is there a sub-quadratic time algorithm for Matrix
Completion\\
with a sub-linear dependence on the condition number?
\end{center}
In fact, eliminating the polynomial dependence on the condition number was posed explicitly as an open 
problem in the context of Alternating Minimization by Jain, Netrapalli and Sanghavi~\cite{JainNS13}.

In this work, we resolve the question in the affirmative. Specifically, we
design a new variant of Alternating Minimization that 
achieves a \emph{logarithmic} dependence on the condition number 
while retaining the fast running
time of the standard Alternating Minimization framework. This is an
exponential improvement in the condition number compared with all subquadratic
time algorithms for Matrix Completion that we are aware of. Our algorithm works even in the noisy 
Matrix Completion setting and under standard assumptions---specifically, the same
assumptions that support theoretical results for the nuclear norm. That is, we assume
that the first $k$ singular vector of the matrix span an incoherent subspace
and that each entry of the matrix is revealed independently with a certain
probability. While strong, these assumptions led to an interesting theory of
Matrix Completion and have become a de facto standard when comparing
theoretical guarantees.

\subsection{Our Results}

For the sake of exposition we begin by explaining our results in the
\emph{exact} Matrix Completion setting, even though our results here are a
direct consequence of our theorem for the noisy case.
In the exact problem
the goal is to recover an unknown rank~$k$ matrix~$M$ from a subsample~$\Omega \subset [n] \times [n]$ of its
entries where each entry is included independently with probability~$p.$ 
We assume that the unknown matrix~$M=U\Lambda U^\trans$ is a 
symmetric $n\times n$ matrix with nonzero singular values $\sigma_1\ge\dots\ge
\sigma_k>0.$ Following~\cite{Hardt13b}, our result generalizes straightforwardly to rectangular matrices.
To state our result we need to define the \emph{coherence} of the subspace
spanned by~$U.$ Intuitively, the coherence controls how large the projection
is of any standard basis vector onto the space spanned by~$U.$ Formally, for a $n\times k$ matrix $U$ with orthonormal columns,
we define the coherence of $U$ to be
\[
\mu(U) = \max_{i\in[n]}\frac nk\|e_i^\trans U\|_2^2\mcom
\]
where $e_1,\dots,e_n$ is the standard basis of $\R^n.$ Note that this
parameter varies between~$1$ and $n/k.$ With this definition, we can state the
formal sample complexity of our algorithm. 

We show that our algorithm outputs a low-rank factorization $XY^\trans$ such
that with high probability $\|M-XY^\trans\|^2\le\eps\opnorm{M}$ provided 
that the expected size of $\Omega$ satisfies
\begin{equation}\equationlabel{exact-result}
pn^2 =
O\left(nk^c\mu(U)^2\log\left(\frac{\sigma_1}{\sigma_k}\right)\log^2\left(\frac
n{\eps}\right)\right) \mper
\end{equation}
Here, the exponent $c>0$ is bounded by an absolute constant. While we did not
focus on minimizing the exponent, our results imply that the value of $c$ can
be chosen smaller if the singular values of $M$ are well-separated. The formal
statement follows from \theoremref{main}. A notable advantage of our algorithm
compared to several fast algorithms for Matrix Completion is that the
dependence on the error $\eps$ is only poly-logarithmic. This linear convergence rate makes 
near exact recovery feasible with a small number of steps.

We also show that the running time of our algorithm is bounded by
$\tilde O(\poly(k)pn^2).$ That is, the running time is nearly linear in the
number of revealed entries except for a polynomial overhead in~$k$.
For small values of $k$ and $\mu(U),$ the total running time is
nearly linear in~$n.$

\paragraph{Noisy Matrix Completion.}
We now discuss our more general result that applies to the noisy or robust
Matrix Completion problem. Here, the unknown matrix
is only close to low-rank, typically in Frobenius norm. 
Our results apply to any matrix of the form
\begin{equation}\label{eq:easydecomp}
A = M + N = U \Lambda U^\trans + N,
\end{equation}
 where $M=U\Lambda
U^\trans$ is a matrix of rank~$k$ as before and $N=(I-UU^\trans)A$ is the part
of $A$ not captured by the dominant singular vectors. We note that $N$ can be an
arbitrary deterministic matrix. The assumption that we will make is 
that~$N$ satisfies the following incoherence conditions:
\begin{equation}\equationlabel{muN}
\max_{i\in[n]}\twonorm{e_i^\trans N}^2\le \frac{\mu_N}n\cdot \min\inset{ \fronorm{N}^2, \sigma_k^2}
\quad\text{and}\quad
\max_{i,j \in N} |N_{ij}| \leq \frac{ \mu_N \fronorm{N} }{n}\mper
\end{equation}
%
%
Recall that $e_i$ denotes the $i$-th standard basis vector so that $\twonorm{e_i^\trans N}$ is
the Euclidean norm of the $i$-th row of $N.$ 
The conditions state no entry of $N$ should be too large compared to 
the norm of the corresponding row in~$N,$ and no row of~$N$ should be too large 
compared to~$\sigma_k.$
Our bounds will be in terms of a combined coherence parameter $\mustar$ satisfying
\begin{equation}\label{eq:defofmustar}
 \mustar \geq \max\Set{\mu(U),\mu_N}.
\end{equation}
We show that our algorithm outputs a rank $k$ factorization $XY^\trans$ such
that with high probability 
\[\|A-XY^\trans\| \le \eps\opnorm{M} + (1+o(1))\|N\|,\]
where $\opnorm{\cdot}$ denotes the spectral norm.
It follows from our argument that we can have the same guaranteee in Frobenius norm as well.
To achieve the above bound we show that it is sufficient to have an expected
sample size
\begin{equation}
pn^2
= 
O\left(n\cdot\poly(k/\gamma_k)(\mustar)^2\log\left(\frac{\sigma_1}{\sigma_k}\right)
\left(\log^2\left(\frac n{\eps}\right)
+ \inparen{\frac{\|N\|_F}{\eps\fronorm{M}}}^2\right)
\right) 
\mper
\end{equation}
Here, $\gamma_k=1-\sigma_{k+1}/\sigma_k$ indicates the separation between the
singular values $\sigma_k$ and $\sigma_{k+1}.$ 
The theorem is a strict generalization of
the noise-free case, which we recover by setting $N=0$ and hence
$\gamma_k=1.$
The formal statement is \theoremref{main}. Compared to our noise-free bound above, there are two new parameters that enter the sample
complexity. The first one is~$\gamma_k.$ The second is the term
$\|N\|_F/\eps\|M\|_F.$ To interpret this quantity, suppose that  
that $A$ has a good low-rank approximation in Frobenius norm: formally, 
$\|N\|_F\le \eps\|A\|_F$ for $\eps\le 1/2.$   Then it must also be the case that
$\|N\|_F/\eps\le 2\|M\|_F.$ Our algorithm then finds a good
rank~$k$ approximation with at most $O(\poly(k)\log(\sigma_1/\sigma_k)(\mustar)^2n)$ samples
assuming $\gamma_k=\Omega(1).$ 
Thus, in the case that $A$ has a good rank~$k$ approximation in Frobenius norm and that $\sigma_k$ and $\sigma_{k+1}$ are well-separated, our bound recovers the noise-free bound up to a constant factor.

For an extended discussion of related work see \sectionref{related}. We
proceed in the next section 
with a detailed proof overview and a description of our notation.

\section{Preliminaries}
In this section, we will give an overview of our proof, give a more in-depth survey of previous work, and set notation.
\subsection{Technical Overview}
As the proof of our main theorem is somewhat complex we will begin with an
extensive informal overview of the argument. In order to understand our main algorithm, 
it is necessary to understand the basic Alternating Minimization algorithm first.

\paragraph{Alternating Least Squares.}
Given a subsample~$\Omega$ of entries drawn from an unknown matrix~$A,$
Alternating Minimization starts from a poor approximation $X_0Y_0^\trans$ to
the target matrix and iteratively refines the approximation by fixing one of
the factors and minimizing a certain objective over the other factor. Here,
$X_0,Y_0$ each have $k$ columns where $k$ is the target rank of the
factorization.  The least squares objective is the typical choice. In this
case, at step~$\ell$ we solve the optimization problem
\[
X_\ell = \arg\min_X\sum_{(i,j)\in\Omega} \left[A_{ij} - (XY_{\ell-1}^\trans)_{ij}\right]^2.
\]
This optimization step is then repeated with $X_\ell$ fixed in order to
determine $Y_\ell.$ 
Since we assume without loss of generality that $A$ is symmetric these steps can be combined into one
least squares step at each point.  What previous work exploited is that this
Alternating Least Squares update can be interpreted as a noisy power method
update step. That is, $Y_\ell = AX_{\ell-1}+G_\ell$ for a noise matrix $G_\ell$.
In this view, the convergence of the algorithm can be controlled by $\|G_\ell\|,$ the spectral norm of
the noise matrix. To a rough approximation, this spectral norm initially behaves 
like $O(\sigma_1/\sqrt{pn})$, ignoring factors of $k$ and $\mu(U).$ 
Since we would like to discover singular vectors
corresponding to singular values of magnitude $\sigma_k,$ we need that the
error term satisfies $\|G_\ell\|\ll \sigma_k$:  otherwise we cannot rule out
that the noise term wipes out any correlation between $X$ and the $k$-th
singular vector. In order to achieve this, we would need to set 
$pn=O((\sigma_1/\sigma_k)^2)$ and this is where a quadratic dependence on
the condition number arises.  This is not the only reason for this dependence:
Alternating Minimization seems to exhibit a linear convergence rate only once
$X_\ell$ is already ``somewhat close'' to the desired subspace~$U.$  This is
why typically the algorithm is initialized with a truncated SVD of the matrix
$P_\Omega(A)$ where $P_\Omega$ is the projection onto the subsample~$\Omega.$
We again face the issue that $\|A-P_\Omega(A)\|$ behaves roughly like
$O(\sigma_1/\sqrt{pn})$ and so we run into the same problem here as well.

A natural idea ot fix these problems is the so-called \emph{deflation}
approach. If it so happens that $\sigma_1\gg \sigma_k,$ then there must be an
$r< k$ such that $\sigma_1\approx\sigma_r\gg \sigma_k.$ In this case, we can
try to first run Alternating Minimization with $r$ vectors instead of $k$
vectors. This results in a rank $r$ factorization $XY^\trans.$ We then
subtract this matrix off of the original matrix and continue with
$A'=A-XY^\trans.$ This approach was in particular suggested by Jain et
al.~\cite{JainNS13} to eliminate the condition number dependence.
Unfortunately, as we will see next, this approach runs into serious issues.

\paragraph{Why standard deflation does not work.}
Given any algorithm \textsc{NoisyMC} for noisy
matrix completion, whose performance depends on the condition number of $A$,
we may hope to use \textsc{NoisyMC} in a black-box way to obtain a deflation-based algorithm
which does not depend on the condition number, as follows.  Suppose that we
know that the spectrum of $A$ comes in blocks, 
\[
\sigma_1 = \sigma_2 = \ldots
= \sigma_{r_1} \gg \sigma_{r_1 + 1} = \sigma_{r_1 +2} = \cdots = \sigma_{r_2}
\gg \sigma_{r_2+1} = \cdots
\] 
and so on.  We could imagine running
\textsc{NoisyMC} on $P_\Omega(A)$ with target rank $r_1$, to obtain an
estimate $M^{(1)}$.  Then we may run \textsc{NoisyMC} again on  $P_\Omega(A -
M^{(1)}) = P_\Omega(A) - P_\Omega(M^{(1)})$ with target rank $r_2-r_1$, to
obtain $M^{(2)}$, and so on.  At the end of the day, we would hope to
approximate $A \approx M^{(1)} + M^{(2)} + \cdots$.  Because we are
focusing only on a given ``flat" part of the spectrum at a time, the
dependence of \textsc{NoisyMC} on the condition number should not matter.  
A major problem with this approach is that the error builds up rather quickly.  More
precisely, any matrix completion algorithm run on $A$ with target rank $r_1$
must have error on the order of $\sigma_{r_1 + 1}$ since this is the spectral
norm of the ``noise part'' that prevents the algorithm from converging
further. Therefore, the matrix $A - M^{(1)}$ might now have $2r_1$
problematic singular vectors corresponding to relatively 
large singular values, namely those vectors arising from the residuals of the
first $r_1$ singular vectors, as well as those arising from the approximation
error. This multiplicative blow-up makes it difficult to ensure convergence. 

\paragraph{Soft deflation.}
The above intuition may make a ``deflation''-based argument seem hopeless. We
instead use an approach that looks similar to deflation but makes an important
departure from it. Intuitively, our algorithm is a single execution of
Alternating Minimization. However, we dynamically grow the number of vectors
that Alternating Minimization maintains until we've reached $k$ vectors. At
that point we let the algorithm run to convergence. More precisely, the
algorithm proceeds in at most $k$ epochs. Each epoch roughly proceeds as
follows:
\begin{description}
\item[Inductive Hypothesis:]
At the beginning of epoch $t,$ the
algorithm has a rank $r_{t-1}$ factorization $X_{t-1} Y_{t-1}^\trans$ that has converged
to within error $\sigma_{r_{t-1}+1}/100.$ At this point, the $(r_{t-1}+1)$-th singular
vector prevents further convergence. 
\item[Gap finding:]
What can we say about the matrix $A_t =
A-X_{t-1}Y_{t-1}^\trans$ at this point? We know that the first $r_{t-1}$ singular vectors of $A$ are
removed from the top of the spectrum of $A_t.$ Moreover, each of the remaining
singular vectors in $A$ is preserverd so long as the corresponding singular value
is greater than $\sigma_{r_{t-1}+1}/10.$ This follows from perturbation
bounds and we ignore a polynomial loss in~$k$ at this point.  Importantly, the
top of the spectrum of $A_t$ corresponds is correlated with the next block of
singular vectors in~$A.$ This motivates the next step in epoch $t,$ which is
to compute the top $k-r_{t-1}$ singular vectors of $A_t$ up to an approximation
error of $\sigma_{r_{t-1}+1}/10.$ Among these singular vectors we now identify a
gap in singular values, that is we look for a number
 $d_t$ such that $\sigma_{r_{t-1}+d_t}\le\sigma_{r_{t-1}+1}/2.$ 
\item[Alternating Least Squares:]
At this point we have identified a new block of $d_t$ singular vectors and we
arrange them into an orthognormal matrix $P_t\in\R^{n\times d_t}.$ We can now
argue that the matrix $W=[ X_{t-1} | P_t]$ is close (in principal angle) to the
first $r_{t}=r_{t-1}+d_t$ singular vectors of $A.$ What this means is that
$W$ is a good initializer for the Alternating Minimization
algorithm which we now run on $W$ until it converges to a rank $r_{t}$
factorization $X_{t}Y_{t}^\trans$ that satisfies the induction hypothesis
of the next epoch. 
\end{description}
We call this algorithm \Deflate.  The crucial difference to the deflation
approach is that we always run Alternating Minimization on a subsampling $P_\Omega(A)$ of the original matrix $A$.
We only ever compute a deflated matrix $P_\Omega(A-XY^\trans)$
for the purpose of initializing the next epoch of the algorithm. This prevents
the error accumulation present in the basic deflation approach.

This simple description glosses over many details and there are a few
challenges to be overcome in order to make the idea work.
For example, we have not said how to determine the appropriate ``gaps" $d_t$.  This
requires a little bit of care.  Indeed, these gaps might be quite small: if
the (additive) gap between $\sigma_{r}$ and $\sigma_{r+1}$ is on the order of,
say, $\frac{ \log^2(k) }{k} \sigma_r$, for all $r \leq k$, then the condition
number of the matrix may be super-polynomial in $k$, a price we are not
willing to pay.  Thus, we need to be able to identify gaps between $\sigma_r$
and $\sigma_{r+1}$ which are on the order of $\sigma_r/k$.  To do this, we
must make sure that our estimates of the singular values of $A - \Xprev\Yprev^\trans$
are sufficiently precise.

%
%

\paragraph{Ensuring Coherence.}
Another major issue that such an algorithm faces is that of coherence.  As mentioned
above, incoherence is a standard (and necessary) requirement of matrix
completion algorithms, and so in order to pursue the strategy outlined above,
we need to be sure that the estimates $X_{t-1}$ stay incoherent.  
For our first ``rough estimation" step, our algorithm carefully
truncates (entrywise) its estimates, in order to preserve the incoherence
conditions, without introducing too much error. In particular, we cannot reuse
the truncation analysis of Jain et al.~\cite{JainNS13} which incurred a
dependence on the condition number.
Coherence in the Alternating Minimization step is handled
by the algorithm and analysis of~\cite{Hardt13b}, upon which
we build. Specifically, Hardt used a form of regularization by noise addition 
called \SmoothGS, as well as an extra step which involves taking medians, 
which ensures that various iterates of Alternating
Minimization remain incoherent.
\subsection{Further Discussion of Related Work}
\sectionlabel{related}

Our work is most closely related to recent works on convergence bound for
Alternating Minimization~\cite{Keshavan12,JainNS13,GunasekarAGG13,Hardt13}. 
Our bounds are in general incomparable. We achieve an exponential improvement in the condition
number compared to all previous works, while losing polynomial factors in~$k$.
Our algorithm and analysis crucially builds on~\cite{Hardt13b}. In particular
we use the version and analysis of Alternating Minimization derived in that
work more or less as a black box. We note that the analyses of Alternating Minimization in other 
previous works would not be sufficiently strong 
to be used in our algorithm. In particular, the use of noise addition to
ensure coherence already gets rid of one source of the condition number that
all previous papers incur. 

We are not aware of a fast nuclear norm solver that has theoretical guarantees
that do not depend polynomially on the condition number. The work of Keshavan
et al.~\cite{KeshavanMO10,KeshavanMO10b} gives another alternative to nuclear
norm minimization that has theoretical guarantees. However, these bounds have
a quartic dependence on the condition number.
We are not aware of any fast nuclear norm solver with theoretical guarantees
that do not depend polynomially on the condition number. The work of Keshavan
et al.~\cite{KeshavanMO10,KeshavanMO10b} gives another alternative to nuclear
norm minimization that has theoretical guarantees. However, these bounds have
a quartic dependence on the condition number.
There are a number of fast algorithms for matrix completion: for example, based on (Stochastic)
Gradient Descent~\cite{RechtR13}; (Online) Frank-Wolfe~\cite{JaggiS10,HazanK12}; or CoSAMP~\cite{LeeB10}.
However, the theoretical guarantees for these algorithms are typically in terms of the 
error on the observed entries, rather than on the error between the recovered matrix and the unknown matrix itself.
For the matrix completion problem, convergence on observations does not imply convergence on the entire matrix.\footnote{
For some matrix recovery problems---in particular, those where the observations obey a rank-restricted isometry property---convergence on the observations is enough to imply convergence on the entire matrix.  However, for matrix completion, the relevant operator does \em not \em satisfy this condition~\cite{CandesR09}.}
Further, these algorithms typically have polynomial, rather than logarithmic, dependence on the
accuracy parameter $\eps$.  Since setting $\eps \approx \sigma_k/\sigma_1$ is required in order
to accurately recover the first $k$ singular vectors of $A$, a polynomial dependence in $\eps$ implies a polynomial dependence on the condition number.

\subsection{Notation}
For a matrix $A$, $\opnorm{A}$ denotes the spectral norm, and $\fronorm{A}$
the Frobenius norm.  We will also use $\infnorm{A} = \max_{i,j}|A_{i,j}|$ to mean the entry-wise $\ell_\infty$ norm.
For a vector $v$, $\twonorm{v}$ denotes the $\ell_2$
norm.  Throughout, $C,C_0,C_1,C_2,\ldots$ will denote absolute constants, and $C$ may change from instance to instance.
We also use standard asymptotic notation $O(\cdot)$ and $\Omega(\cdot)$, and we occasionally use $f \lesssim g$ (resp. $\gtrsim$) 
	to mean $f = O(g)$ (resp. $f = \Omega(g)$) to remove notational clutter.  Here, the asymptotics are taken as $k,n \to \infty$.
For a matrix $X \in \R^{n \times k}$, $\mathcal{R}(X)$ denotes
the span of the columns of $X$, and $\PiX$ denotes the orthogonal projection
onto $\mathcal{R}(X)$.  Similarly, $\PiXPerp$ denotes the projection onto
$\mathcal{R}(X)^\perp$.
For a set random $\Omega \subset [n] \times [n]$ and a matrix $A \in \R^{n \times n}$, we define the (normalized) projection operator $P_\Omega$ as 
\[P_\Omega(A) := \frac{n^2}{\EE |\Omega|} \sum_{(i,j) \in \Omega} A_{i,j} e_ie_j^\trans \] 
to the be matrix $A$, restricted to the entries indexed by
$\Omega$ and renormalized.

\subsubsection{Decomposition of $A$}
Our algorithm, and its proof, will involve choosing
a sequence of integers $r_1 < \cdots < r_t \leq k$, which will mark the
significant ``gaps'' in the spectrum of $A$.
Given such a sequence, we will decompose~$A$ as
\begin{equation}\label{eq:decomp}
 A = \Mt + \Nt = M^{(1)} + M^{(2)} + \cdots + M^{(t)} + \Nt, 
\end{equation}
where $\Mt$ has the spectral decomposition $\Mt = \xlt{U} \Lambda_{(\leq t)} (\xlt{U})^\trans$ and 
$\Lambda_{(\leq t)}$ contains the eigenvalues corresponding to
singular values $\sigma_1\ge\cdots\ge\sigma_{r_t}$.  We may decompose $\Mt$ as
the sum of $M^{(j)}$ for $j=1\dots t,$  where each $M^{(j)}$ has the spectral decomposition
$M^{(j)} = U^{(j)} \Lambda_j \inparen{U^{(j)}}^\trans$
corresponding to the singular values $\sigma_{r_{j-1}+1},\ldots,\sigma_{r_j}$.
Similarly,  the matrix $\Nt$ may be written as
$\Nt = (\Vt) \Lambda_{(>t)} (\Vt)^\trans,$
and contains the singular values $\sigma_{r_t + 1}, \ldots, \sigma_n$. 
Eventually, our algorithm will stop at some maximum $t = T$, for which $r_t = k$, and we will have
$A = M + N = M^{(\leq T)} + N_T$ as in \eqref{eq:easydecomp}.
We will use the notation $U^{(\leq j)}$ to denote the concatenation
\begin{equation}\label{eq:leq}
 U^{(\leq j)} = [ U^{(1)} | U^{(2)} | \cdots | U^{(j)} ]. 
\end{equation}
Observe that this is consistent with the definition of $\xlt{U}$ above.
Additionally, for a matrix $X \in \R^{n \times r_t}$, we will write
$ X = [ X^{(1)} | X^{(2)} | \cdots | X^{(t)} ], $
where $X^{(j)}$ contains the $r_{j-1} +1,\ldots, r_j$ columns of $X$, and we will write
$ \xlj{X} = [ X^{(1)} | X^{(2)} | \cdots |X^{(j)} ].$
Occasionally, we will wish to use notation like $U^{(\leq r)}$ to denote the first $r$ columns (rather than the first $r_r$ columns).  This will be pointed out when it occurs.


For an index $r \leq n$, we quantify the gap between $\sigma_r$ and $\sigma_{r+1}$ by
\begin{equation}\label{eq:gammar}
 \gamma_r := 1 - \frac{ \sigma_{r+1}}{\sigma_r}. 
\end{equation}
and we will define
\begin{equation}\label{eq:gamma}
 \gamma :=  \min \inset{ \gamma_r \suchthat r \in [n], \gamma_r \geq \frac{1}{4k} }.
\end{equation}
By definition, we always have $\gamma \geq 1/4k$; for some matrices $A$, it may be much larger, and this will lead to improved bounds.
Our analysis will also depend on the ``final" gap quantified by $\gamma_k$, whether or not it is larger than $1/4k$.  To this end, we define
\begin{equation}\label{eq:gammastar}
\gamma^* := \min \inset{ \gamma, \gamma_k }.
\end{equation}

\section{Algorithms and Results}\label{sec:algs}

%
In Algorithm \ref{alg:deflate} we present our main algorithm \Deflate. It uses
several subroutines that are presented in \sectionref{subroutines}.
\newcommand{\comm}[1]{{\color{DarkBlue}\hspace{5mm}\small{// #1}}}
\begin{algorithm}
\KwIn{
Target dimension $k$;
Observed set of indices
$\Omega\subseteq [n]\times[n]$ of an unknown symmetric matrix $A\in\R^{n\times
n}$ with entries $P_\Omega(A)$;
Accuracy parameter $\eps$;
Noise parameter $\Delta$ with $\|A - A_k\|_F\leq \Delta$;
Coherence parameter $\mu^*$, satisfying \eqref{eq:defofmustar}, and a
parameter $\mu_0$;
Probabilities $p_0$ and $p_t,p_t'$ for $t = 1,\ldots ,k$;
Number of
iterations $L_t \in\mathbb{N},$ for $t=1,\ldots,k$ runs of \AltLS, and a parameter $\Smax \in \mathbb{N}$ for \AltLS, and a number of iterations $L$ for runs of \SI.
}
Let $p = \sum_t (p_t + p_t')$.

Break $\Omega$ randomly into $2k + 1$ sets, $\Omega_0$ and $\Omega_1,\Omega_1', \ldots, \Omega_k, \Omega_k'$,
so that $\EE |\Omega_t| = \frac{ p_t}{p}|\Omega|$ and $\EE |\Omega'_t| = \frac{p_t'}{p}|\Omega|$  (See Remark \ref{rem:subsample}).

$ s_0 \gets  \opnorm{ P_{\Omega_0}(A) }$ \comm{ Estimate $\sigma_1(A)$}

Initialize $X_0=Y_0=0,$ $r_0=0$

\For{ $t=1\dots k$}
{
$\tau_t \gets \frac{ \mu^* }{np_t} \inparen{ 2k s_{t-1} + \Delta}$ \nllabel{line:start}

$T_t \gets \text{\textsc{Truncate}}\inparen{ P_{\Omega_t}(A) - P_{\Omega_t}(\Xprev \Yprev^\trans) , \tau_t}$
\comm{\textsc{Truncate}$(M,c)$ truncates $M$ so that $|M_{ij}| \leq c$}

$\widetilde{U}_t, \vec{\tilde{\sigma}} \gets \text{\SI}( T_t, k-r_{t-1}, L )$\nllabel{line:svd}
\comm{Estimate the top $k - r_{t-1}$ spectrum of $T_t$.}

{\bf If} $\widetilde{\sigma}_1  < 10 \eps s_0 $ {\bf then return} $X_{t-1},Y_{t-1}$ \nllabel{line:done}

$d_t\leftarrow \min\inset{ i\le k- r_{t-1}\suchthat  \sigma_{i+1}(\widetilde{T}_t) \leq \inparen{ 1 - \frac{1}{4k}} \sigma_i(\widetilde{T}_t)} \cup \inset{k-r_{r-1}}$\nllabel{eq:dt}

$r_t \leftarrow r_{t-1} +d_t$
\comm{$r_t$ is an estimate of the next ``gap" in the spectrum of $A$}

$s_t \gets \widetilde{\sigma}_{d_t}$
\comm{ $s_t$ is an estimate of $\sigma_{r_t}(A)$}

$\widetilde{Q}_t \gets \inparen{ \widetilde{U}_t }^{(\leq d_t)}$
\comm{Keep the first $d_t$ columns of $\widetilde{U}_t$}

$\overline{Q}_t \gets \text{\textsc{Truncate}}\inparen{\widetilde{Q}_tB, 8\sqrt{ \frac{ \mu^* log(n) }{n}}}$ \comm{where $B \in \R^{n \times n}$ is a random orthonormal matrix.}


$W_t \leftarrow \text{\GS}( [ X_{t-1} \mid \overline{Q}_t ] )$ \comm{$W_t$ is a rough estimate of $U^{(\leq t)}$}\nllabel{line:wt}

$\mu_t \gets \inparen{ \sqrt{\mu_0} + (t-1) \sqrt{ \mu^* k } }^2$\nllabel{line:start2} 

$(X_t,Y_t) \leftarrow \text{\AltLS}(A,\Omega_t',R_0=W_t,L=L_t,\Smax=\Smax, k=r_t, \zeta=\eps s_0 k^{-5}, \mu=\mu_t)$ \nllabel{line:altls}
\comm{$X_t$ is a good estimate of $U^{(\leq t)}$}

{\bf If} $r_t \ge k$ {\bf then return}$(X_t,Y_t)$

}
\KwOut{ Pair of matrices $(X,Y).$}
\caption{\Deflate: Approximates an approximately low-rank matrix from a few entries. } 
\label{alg:deflate}
\end{algorithm}

\begin{remark}\label{rem:subsample} In the Matrix Completion literature, the
most common assumption on the distribution of the set $\Omega$ of observed
entries is that each index $(i,j)$ is included independently with some
probability $p$.  Call this distribution $\cD(p)$.  In order for our results to
be comparable with existing results, this is the model we adopt as well.
However, for our analysis, it is much more convenient to imagine that $\Omega$
is the union of several subsets $\Omega_t$, so that the $\Omega_t$ themselves
follow the distribution $\cD(p_t)$ (for some probability $p_t$, where $\sum_t
p_t = p$), and so that all of the $\Omega_t$ are independent.  Algorithmically,
the easiest thing to do to obtain subsets $\Omega_t$ from $\Omega$ is to
partition $\Omega$ into random subsets of equal size.  However, if we do this,
the subsets $\Omega_t$ will not follow the right distribution; in particular
they will not be independent.  For theoretical completeness, we show in
Appendix \ref{app:subsample} (Algorithm \ref{alg:splitup}) 
how to split up the set $\Omega$ in the correct
way.  More precisely, given $p_t$ and $p$ so that $\sum_t p_t = p$, we show how
to break $\Omega \sim\cD(p)$ into (possibly overlapping) subsets $\Omega_t$, so that
the $\Omega_t$ are independent and each $\Omega_t \sim \cD(p_t)$.  
\end{remark}

\subsection{Overview of Subroutines}
\sectionlabel{subroutines}

\Deflate uses a number of subroutines that we outline here before explicitly
presenting them:
\begin{itemize}
\item \AltLS (Algorithm \ref{alg:altmin}) is the main Alternating Least Squares procedure that was given
and analyzed in~\cite{Hardt13b}. We use this algorithm and its analysis.
\AltLS by itself has a quadratic dependence on the condition number which is
why we can only use it as a subroutine.
\item \SmoothGS (Algorithm \ref{alg:smoothGS}) is a subroutine of \AltLS which is used to control the
coherence of intermediate solutions arising in \AltLS. Again, we reuse the
analysis of \SmoothGS from~\cite{Hardt13b}. \SmoothGS orthonormalizes its
input matrix after adding a Gaussian noise matrix. This step allows tight
control of the coherence of the resulting matrix. We defer the description of
\SmoothGS to Section \ref{sec:proofofmainlemma} where we need it for the first
time. 
\item \SI is a standard textbook version of the Subspace Iteration algorithm
(Power Method). We use this algorithm as a fast way to approximate the top
singular vectors of a matrix arising in \Deflate. 
We use only standard properties of \SI in our analysis. For this
reason we defer the description and analysis of \SI to \sectionref{subsit}.
\end{itemize}

\begin{algorithm}
\KwIn{Number of
iterations $L\in\mathbb{N},$ parameter $\Smax \in \mathbb{N}$, target dimension $k,$ observed set of indices
$\Omega\subseteq [n]\times[n]$ of an unknown symmetric matrix $A\in\R^{n\times
n}$ with entries $P_\Omega(A),$ initial orthonormal matrix $R_0\in\R^{n\times k},$
and parameters $\zeta, \mu$}

Break $\Omega$ randomly into sets $\Omega_1,\ldots, \Omega_L$ with equal expected sizes.  
(See Remark \ref{rem:subsample}).

\For{ $\ell = 1$ to $L$}{

Break $\Omega_\ell$ randomly into subsets $\Omega_\ell^{(1)}, \dots, \Omega_\ell^{(T)}$ with equal expected sizes.

\For {$s = 1$ to $\Smax$}
{
	$S^{(s)}_\ell \leftarrow \arg\min_{S\in\R^{n\times
k}}\fronorm{P_{\Omega_\ell}(A-R_{\ell-1}S^\trans)}^2$
}
$S_\ell \gets \median_s( S^{(s)}_\ell )$ \comm{ The median is applied entry-wise. }

$R_\ell \leftarrow \text{\SmoothGS}(S_\ell, \zeta, \mu)$

}
\KwOut{Pair of matrices $(R_{L-1},S_L)$ }
\caption{ $\text{\AltLS}(P_\Omega(A),\Omega,R_0,L,\Smax,k,\zeta,\mu)$ (Smoothed-Median-Alternating Least Squares) }
\label{alg:altmin}
\end{algorithm}

\subsection{Statement of the main theorem}
Our main theorem is that, when the number of samples is $\poly(k) n$, \Deflate returns a good estimate of $A$, with at most logarithmic dependence on the condition number.
\begin{theorem}\label{thm:mainthm}
\theoremlabel{main}
There is a constant $C$ so that the following holds.
Let $A \in \R^{n \times n}$, $k \leq n$, and write $A = M + N$, where $M$ is the best rank-$k$ approximation to $A$.
Let $\gamma, \gamma^*$ be as in \eqref{eq:gamma}, \eqref{eq:gammastar}.
Choose parameters for Algorithm \ref{alg:deflate} so that $\eps > 0$ and
\begin{itemize}
\item $\mu^*$ satisfies \eqref{eq:defofmustar} and 
$\mu_0 \geq \frac{ C }{(\gamma^*)^2}\inparen{ \mu^* \inparen{ k + \inparen{ \frac{ k^4\Delta }{\eps \sigma_1}}^2 } + \log(n) }$ 
\item $\Delta \geq \fronorm{N}$ 
\item $L_t \geq \frac{C}{\gamma^*} \log\inparen{\frac{k\sigma_{r_t}}{\sigma_{r_t+1} +\eps\sigma_1}} , $ and $L \geq C k^{7/2} \log(n)$
\item $\Smax \geq C\log(n)$.
\end{itemize}
There is a choice of $p_t, p_t'$ (given in the proof below) so that
\[ p = \sum p_t + \sum p_t' \leq C \frac{ k^{9}}{(\gamma^*)^3 n }  \log\inparen{ k\cdot  \frac{\sigma_1}{\sigma_k + \eps\sigma_1}  } \inparen{ 1 + \inparen{\frac{ \Delta }{\eps\opnorm{M}}}^2 } \inparen{ \mu_0 + \mu^*k\log(n) } \log^2(n) \] 
so that the following holds.

Suppose that each element of $[n] \times [n]$ is included in $\Omega$ independently with probability $p$.
Then the matrices $X, Y$ returned by \Deflate satisfy
with probability at least $1 - 1/n,$
\[ \opnorm{A - XY^\trans} \leq \inparen{1 + o(1)}\opnorm{N} + \eps \opnorm{M}.\]
\end{theorem}

\begin{remark}[Error guarantee]
The guarantee of $\opnorm{A - XY^\trans} \leq \inparen{1 + o(1)} \opnorm{N} + \eps\opnorm{M}$ is what naturally falls out of our analysis: the natural value for the $o(1)$ term is polynomially small in $k$.  It is not hard to see in the proof that we may make this term as small as we like, say, $(1 + \alpha)\opnorm{N}$, by paying a logarithmic penalty $\log(1/\alpha)$ in the choice of $p$.
It is also not hard to see that we may have a similar conclusion for the Frobenius norm.
\end{remark}

\begin{remark}[Obtaining the parameters]
As written, then algorithm requires the user 
to know several parameters which depend on the unknown matrix $A$.  For some parameters, these requirements are innocuous.
For example, to obtain $p_t'$ or $L_t$ (whose values are given in Section \ref{sec:puttingittogether}), the user is required to have a bound on $\log( \sigma_{r_t} /
\sigma_{r_t + 1})$.  Clearly, a bound on the condition number of $A$ will
suffice, but more importantly, the estimates $s_t$ which appear in Algorithm
\ref{alg:deflate} may be used as proxies for $\sigma_{r_t}$, and so the
parameters $p_t'$ can actually be determined relatively precisely on the fly.
For other parameters, like $\mu^*$ or $k$, we assume that the user has a good estimate from other sources.  While this is standard in the Matrix Completion literature, we acknowledge that these values may be difficult to come by. 

\end{remark}

\subsection{Running Time}
\label{sec:runningtime}
The running time of \Deflate is linear in $n$, polynomial in $k$, and logarithmic in the condition number $\sigma_1 / \sigma_k$ of $A$.
Indeed, the outer loop performs at most $k$ epochs, and the nontrivial operations in each epoch are \AltLS, \GS, and \SI.  
All of the other operations (truncation, concatenation) are done on matrices which are either $n \times k$ (requiring at most $nk$ operations) or on the subsampled matrices $P_{\Omega_t}(A)$, requiring on the order of $pn^2$ operations.

Running \SI requires $L = O(k^{7/2}\log(n))$ iterations; each iteration
includes multiplication by a sparse matrix, followed by \GS.  The matrix
multiplication takes time on the order of 
\[p_tn^2 = n\,\poly(k)\,\log(n)\, \inparen{ 1 + \frac{\Delta}{ \eps \sigma_1 }},\] 
the number of nonzero
entries of $A$, and \GS takes time $O(k^2n)$. 
Each time \AltLS is run, it takes $L_t$ iterations, and we will show (following the analysis of \cite{Hardt13b}) that it requires $\poly(k) n \log(n)\log(n/\eps)$ operations per iteration.
Thus, given the choice of $L_t$ in Theorem \ref{thm:mainthm}, the total running time  of \Deflate
on the order of
\[  \tilde{O} \inparen{ n \cdot\poly(k) \cdot \inparen{ 1 + \frac{ \Delta }{ \eps \sigma_1 } }\cdot \log\inparen{ \frac{\sigma_1}{\sigma_k + \eps\sigma_1} }  }, \]
where the $\tilde{O}$ hides logarithmic factors in $n$.

\section{Proof of Main Theorem}
In this section, we prove Theorem \ref{thm:mainthm}.
The proof proceeds by maintaining a few inductive hypotheses, given below, at each epoch.
When the algorithm terminates, we will show that the fact that these hypotheses still hold imply the desired results.
Suppose that at the beginning of step $t$ of Algorithm \ref{alg:deflate},
we have identified some indices $r_1, \ldots, r_{t-1}$, and recovered estimates $\Xprev,\Yprev$ which capture the singular values
$\sigma_1, \ldots, \sigma_{r_{t-1}}$ and the corresponding singular vectors.
The goals of the current step of Algorithm \ref{alg:deflate} are then to
(a) identify the next index $r_t$ which exhibits a large ``gap" in the spectrum, and
(b) estimate the singular values $\sigma_{r_{t-1}+1},\ldots,\sigma_{r_t}$ and the corresponding singular vectors.

Letting $r_t$ be the index obtained by Algorithm \ref{alg:deflate}, 
we will decompose $A = \Mprev +\Nprev = \Mt +\Nt$ as in \eqref{eq:decomp}.
To help keep the notation straight, we include a diagram below, which indicates which singular values of $A$ are included in which matrix.
\begin{center}

\begin{tikzpicture}
\draw[thick] (0,0) -- (5,0);
\draw (0,-.1) -- (0,.1);
\draw (2,-.1) -- (2,.1);
\draw (3,-.1) -- (3,.1);
\draw (5,-.1) -- (5,.1);
\node(z) at (0, -.3) {$0$};
\node(rprev) at (2, -.3) {$r_{t-1}$};
\node(rt) at (3, -.3) {$r_t$};
\node(n) at (5, -.3) {$n$};
\draw [
    decoration={
        brace,
        mirror,
        raise=0cm,
	amplitude=5pt
    },
    decorate
] (rprev.south) -- (rt.south) 
node[pos=0.5,anchor=north,yshift=-0.2cm] {$M^{(t)}$};
\draw [
    decoration={
        brace,
        mirror,
        raise=0cm,
	amplitude=5pt
    },
    decorate
] (rt.south) -- (n.south) 
node[pos=0.5,anchor=north,yshift=-0.2cm] {$N_t$}; 
\draw [
    decoration={
        brace,
        mirror,
        raise=1cm,
	amplitude=8pt
    },
    decorate
] (z.south) -- (rt.south) 
node[pos=0.5,anchor=north,yshift=-1.5cm] {$\Mt$}; 
\draw [
    decoration={
        brace,
        raise=.4cm,
	amplitude=8pt
    },
    decorate
] (z.north) -- (rprev.north) 
node[pos=0.5,anchor=north,yshift=1.2cm] {$\Mprev$}; 
\draw [
    decoration={
        brace,
        raise=.4cm,
	amplitude=8pt
    },
    decorate
] (rprev.north) -- (n.north) 
node[pos=0.5,anchor=north,yshift=1.2cm] {$\Nprev$}; 
\end{tikzpicture}

\end{center}

Following Remark \ref{rem:subsample}, we treat the $\Omega_t$ and $\Omega_t'$ as independent random sets, with each entry included with probability $p_t$ or $p_t'$, respectively.
We will keep track of the \em principal angles \em between the subspaces $\mathcal{R}(\xlj(X_{t-1}))$ and $\mathcal{R}(\xlj(U))$.  More precisely, for matrices $A, B \in \R^{n \times r_j}$ with orthogonal columns, we define
\[ {\sin \theta( A, B )} := \opnorm{ {(A_\perp)^{\trans}} B }. \]

We will maintain the following inductive hypotheses.  At the beginning of epoch $t$ of \Deflate, we assert
\begin{equation}\label{h1}
\sigma_{r_j} {\sin\theta ( \xlj{X_{t-1}},\xlj{U} )} \leq \frac{1}{k^{4}}\inparen{ \sigma_{r_{t-1}+1} + \eps\opnorm{M} } \qquad \forall j\leq t-1 
\tag{H1}
\end{equation}
and
\begin{equation}\label{h2}
\opnorm{ \Mprev - X_{t-1}Y_{t-1}^\trans } \leq \frac{\sigma_{r_{t-1} + 1} + \eps\opnorm{M}}{ \Cerr k^{3} } \tag{H2}
\end{equation}
for some sufficiently large constant $\Cerr$ determined by the proof.
We also maintain that the current estimate $X_{t-1}$ is incoherent:
\begin{equation}\label{h3}
\max_{i \in [n]} \twonorm{ e_i^\trans X_{t-1} } \leq \sqrt{\frac{k}{n}} \inparen{\sqrt{\mu_{0}}(1 + \Cinc/k)^{t-1} + (t-1)16\sqrt{ \mu^* \log(n) } } =: \sqrt{ \frac{k\mu_{t-1}}{n}} \tag{H3}
\end{equation}
for a constant $\Cinc$.
Above,
equation \eqref{h3} defines $\mu_{t-1}$.
Observe that when $t=1$, everything in sight is zero and the hypotheses \eqref{h1}, \eqref{h2},\eqref{h3} are satisfied.
Finally, we assume that the estimate $s_{t-1}$ of $\sigma_{r_{t-1} +1}$ is good. 
\begin{equation}\label{eq:s}
\frac{1}{2} \sigma_{r_{t-1} + 1} \leq  s_{t-1} \leq 2 \sigma_{r_{t-1}+1} \tag{H4}
\end{equation}
The base case for \eqref{eq:s} is handled by the choice of $s_0$ in Algorithm \ref{alg:deflate}.
Indeed, Lemma \ref{lem:matrixbernstein} in the appendix implies that, with probability $1 - 1/\poly(n)$, 
\begin{align*}
\opnorm{A - P_{\Omega_0}(A) } &\lesssim \sqrt{ \frac{ \max_i \twonorm{e_i^\trans A}^2 \log(n) }{p_0}} + \frac{\infnorm{A}\log(n)}{p_0}\\
&\leq  \sqrt{ \frac{ \mu^* (\sqrt{k}\sigma_1 + \Delta)^2\log(n)}{np_0} } + \frac{ \mu^* (k\sigma_1 + \Delta)\log(n) }{np_0 } \\
&\leq \inparen{ \frac{\sigma_1}{2} }\inparen{ \sqrt{\frac{ 4 \mu^* \inparen{ \sqrt{k} + (\Delta/\sigma_1) }^2\log(n) }{np_0}} + \frac{ 2\mu^* \inparen{ k + \Delta/\sigma_1 } \log(n) }{np_0 }},
\end{align*}
where we used 
 the incoherence bounds \eqref{eq:incoherenceofA} and \eqref{eq:incoherenceofA2} in the appendix to bound $\infnorm{A}$ and $\twonorm{e_i^\trans A}$.
Thus, as long as
\begin{equation}\label{eq:p0}
 p_0 \gtrsim \frac{ \mu^*  \log(n)\inparen{ \sqrt{k} + \frac{\Delta}{\sigma_1} }^2 } {n},
\end{equation}
then
\[ \frac{1}{2}\sigma_1  \leq \opnorm{  P_{\Omega_0}( A ) } \leq 2\sigma_1. \]
and so \eqref{eq:s} is satisfied.

Now, 
suppose that \hyps hold.  We break up the inner loop of \Deflate into two main steps.  In the first step, lines \ref{line:start} to \ref{line:wt} in Algorithm \ref{alg:deflate}, the goal is to obtain an estimate $r_t$ of the next ``gap," as well as an estimate $W_t$ of the subspace $U^{(\leq t)}$.
We analyze this step in Lemma \ref{lem:correlate} below.  

\begin{lemma}\label{lem:correlate}
There exists a constants $C,\Cgapconstant$ so that the following holds.
Suppose that 
\[ p_t \geq  \frac{C (\mu^*)^2\log(n) \inparen{ k^2 + \inparen{ \frac{ \Delta}{\eps\opnorm{M}}}^2 }}{n \eps_0^2 } , \]
where
$\eps_0 \leq \frac{1}{4 \Cgapconstant k^{5/2}}.$
Further assume that \hyps hold.
Then with probability at least $1 - 1/n^2$ over the choice of $\Omega_t$ and the randomness in \SI, one of the following statements must hold:
\begin{enumerate}
\item[(a)] Algorithm \ref{alg:deflate} terminates at line \ref{line:done}, and returns $\Xprev,\Yprev$ so that
$\opnorm{A - \Xprev \Yprev^T} \leq C\eps\opnorm{M} $; or
\item[(b)] Algorithm \ref{alg:deflate} does not terminate at line \ref{line:done}, 
and the following conditions hold:
\begin{itemize}
\item The error level $\eps$ has not yet been reached:
\begin{equation}
\label{eq:assumeeps1}
\eps\opnorm{M} \leq \sigma_{r_{t-1}+1}\mper
\end{equation}
\item The index $r_t$ recovered obeys
\begin{equation}\label{eq:gaps}
	\frac{ \sigma_{r_t+1} }{\sigma_{r_t}} \leq 1 - \gamma 
\qquad \text{and} \qquad 
	\frac{ \sigma_{r_{t-1} + 1} }{ \sigma_{r_t}} \leq e \mper
\end{equation}
\item
The matrix $W_t$ has orthonormal columns, and satisfies
\begin{equation}\label{eq:muWt}
 \sin \theta( {W_t}, U^{(\leq t)} ) \leq \frac{1}{k} \qquad
\text{and} \qquad
\max_i \twonorm{ e_i^\trans W_t} \leq \sqrt{ \frac{k\mu_t }{n}} ,
\end{equation}
where $\mu_t$ is defined as in \eqref{h3}.
\item
The estimate $s_t$ satisfies \eqref{eq:s}.
\end{itemize}
\end{enumerate}
\end{lemma}

The proof of Lemma \ref{lem:correlate} is given in Section \ref{sec:correlate}.
In the second part of \Deflate, lines \ref{line:start2} to \ref{line:altls} in Algorithm \ref{alg:deflate}, we run \AltLS, initialized with the subspace $W_t$ returned by the first part of the algorithm.  Lemma \ref{lem:parameterchasing} below shows that \AltLS improves the estimate $W_t$ to the desired accuracy, so that we may move on to the next iteration of \Deflate.

\begin{lemma}\label{lem:parameterchasing}
Assume that the conclusion~(b) of Lemma \ref{lem:correlate} holds, as well as \hyps.
There is a constant $C$ so that the following holds.
Let $\gamma^*$ be as in \eqref{eq:gammastar}.
Suppose that
\[ \mu_t \geq \frac{C}{(\gamma^*)^2 }  \inparen{ \mu^* \inparen{ k + \inparen{ \frac{ k^4 \Delta }{ \eps \sigma_1 }}^2 } + \log(n)  }  \]
and 
\[ p_t' \geq \frac{ C L_t \Smax \cdot  k^{9} \mu_t \log(n) \inparen{ k + \inparen{ \frac{ \Delta }{\eps \opnorm{M}}  }^2}}{(\gamma^*)^2 n }
\quad\text{with}\quad
L_t \geq  \frac{C}{\gamma^*} \log \inparen{\frac{  k\sigma_{r_t} }{ \sigma_{r_t+1} + \eps\opnorm{M} } }  \quad \text{and}\quad \Smax \geq C\log(n).
\]
Then after $L_t$ 
steps of \AltLS with the initial matrix $W_t$, and parameters $\mu_t,\eps$, the following hold with probability at least $1 - 1/n^2$, over the choice of $\Omega_t'$.
\begin{itemize}
\item The inductive hypothesis \eqref{h1} holds for the next round: 
\[ \forall j \leq t, \sigma_{r_j} {\sin\theta( \Xt^{(\leq j)}, U^{(\leq j)} )} \leq
\frac{\sigma_{r_t+1} + \eps\opnorm{M} }{k^{4}}.\]
\item The inductive hypothesis \eqref{h2} holds for the next round: 
\[\opnorm{ \Mt - X_t\tY_t } \leq  \frac{ \sigma_{r_t + 1} + \eps\opnorm{M} }{\Cerr
k^{3}}.\]
\item The inductive hypothesis \eqref{h3} holds for the next round:
$\mu(\Xt) \leq \mu_t.$
\end{itemize}
\end{lemma}
The proof of Lemma \ref{lem:parameterchasing} is addressed in Section \ref{sec:proofofmainlemma}.  

\subsection{Putting it together}\label{sec:puttingittogether}
Theorem \ref{thm:mainthm} now follows using
\ref{lem:correlate} and \ref{lem:parameterchasing}.
First, we choose $\mu_0$ as in the statement of Theorem \ref{thm:mainthm}.
Because $\mu_t \geq \mu_0$ for all $t =1,\ldots, T$,  this implies that $\mu_t$ satisfies the requirements of Lemma \ref{lem:parameterchasing}.
Then, the hypotheses of Lemma \ref{lem:parameterchasing} are implied by the conclusions of the favorable case of Lemma \ref{lem:correlate}.  
Now, a union bound over at most $k$ epochs of \Deflate ensures that with probability at least $1 - \nicefrac{2k}{n^2} \geq 1 - 1/n$, the conclusions of both lemmas hold every round that their hypotheses hold.

If \Deflate terminates with the guarantees (a) of Lemma \ref{lem:correlate},
then $\opnorm{ A - X_T \tY_T } \leq C\eps\opnorm{M}.$
On the other hand, if (b) holds, then Lemma \ref{lem:correlate} implies \eqref{eq:s} and the hypotheses of Lemma \ref{lem:parameterchasing}, and then
Lemma \ref{lem:parameterchasing} implies that with probability $1 - 1/n^2$, the remaining inductive hypotheses \eqref{h1}, \eqref{h2}, and \eqref{h3} for the next round. 

Thus, if the situation (a) above never occurs, then the hypotheses of Lemma \ref{lem:parameterchasing} hold until \Deflate
 terminates because $r_t = k$.  In this case, Lemma \ref{lem:parameterchasing} implies that
\[
\opnorm{A - X_T \tY_T } \leq \opnorm{ M_T - X_T \tY_T} + \opnorm{N_T} 
\leq \frac{ \sigma_{k + 1} + \eps\opnorm{M} }{ \Cerr k^{3} } + \sigma_{k+1} \mper
\]
In either case,
\[ \opnorm{ A - X_t  \tY_T } \leq \opnorm{N}\inparen{1 + \frac{1}{\Cerr k^{3}} } + C\eps\opnorm{M}.\]

Finally, we tally up the number of samples. 
The base case \eqref{eq:p0} required
\[ p_0 \gtrsim \frac{ \mu^*  \log(n)\inparen{ \sqrt{k} +  \frac{\Delta}{\opnorm{M}} }^2 } {n}. \]
Lemma \ref{lem:correlate} required
\[ p_t \gtrsim  \frac{ (\mu^*)^2 k^5 \log(n) \inparen{ k^2 + \inparen{ \frac{ \Delta}{\eps\opnorm{M}}}^2 }}{n } . \]
For Lemma \ref{lem:parameterchasing}, we required, for a sufficiently large constant $C$,
\[ p_t' \gtrsim \frac{ CL_t \Smax \cdot k^{9} \mu_t \log(n) \inparen{ k + \inparen{ \frac{ \Delta }{\eps\opnorm{M}}  }^2}}{(\gamma^*)^2 n }
\quad\text{with}\quad
L_t \geq \frac{C}{  \gamma^* } \log \inparen{ \frac{  k\sigma_{r_t} }{ \sigma_{r_t+1} + \eps\opnorm{M} } }
\quad \text{and}\quad \Smax\geq C\log(n)
. \]
From the definition of $\mu_t$ (in \eqref{h3}), we may bound $\mu_t$ for all $t \leq k$ by
\[ \mu_t \leq C\mu_0 + k\log(n)\mu^* \]
for some constant $C$.
Summing over $t$ gives the result.

\section{Proof of Lemma \ref{lem:correlate}}\label{sec:correlate}
In this section, we prove Lemma \ref{lem:correlate}, which shows that either Algorithm \ref{alg:deflate} hits the
precision parameter $\eps$ and returns, or else produces an estimate $W_t$ for $U^{(\leq t)}$ that is close enough to run \AltLS on.
There are several rounds of approximations between the beginning of iteration $t$ and the output $W_t$.  For the reader's convenience, we include an informal synopsis of the notation in Figure \ref{fig:part1}.
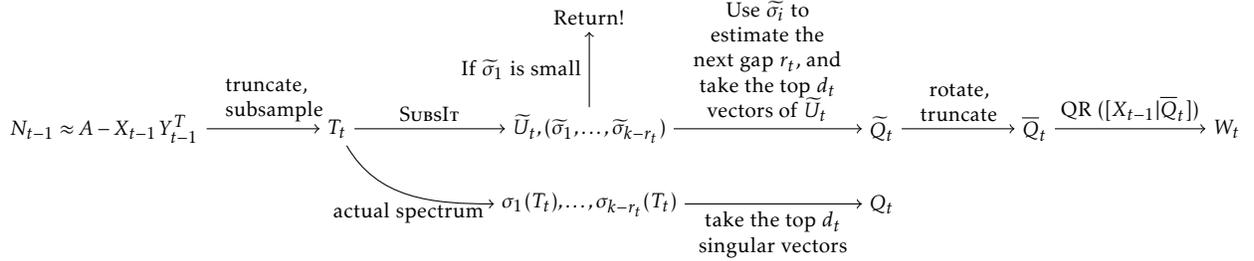
\begin{figure}[h]
\footnotesize
\begin{center}
\begin{tikzpicture}
\node(Nprev) {$\Nprev \approx A - \Xprev \Yprev^\trans$};
\node[right=1.5cm of Nprev](Tt) {$T_t$};
\node[right=2cm of Tt](eigstuff) {$\widetilde{U}_t,(\tilde{\sigma}_1, \ldots, \tilde{\sigma}_{k-r_t})$};
\node[below of=eigstuff](eigstuff2) {$\sigma_1(T_t), \ldots, \sigma_{k-r_t}(T_t)$};
\node[right=2.5cm of eigstuff](Qt) {$\tilde{Q}_t$};
\node[ below of=Qt](Qt2) {$Q_t$};
\node[right=1.5cm of Qt](barQt) {$\overline{Q}_t$};
\node[right=2cm of barQt](Pt) {$W_t$};
\node[above=1cm of eigstuff](done) {Return!};

\draw[->] (eigstuff) to node[left]{If $\tilde{\sigma}_1$ is small} (done);
\draw[->] (Nprev) to node[above]{ \begin{minipage}{1cm} \begin{center}truncate, \\ subsample\end{center} \end{minipage} } (Tt);
\draw[->] (Tt) to node[above]{\SI } (eigstuff);
\draw[->] (eigstuff) to node[above]{ \begin{minipage}{2cm} \begin{center} Use $\tilde{\sigma_i}$ to estimate the next gap $r_t$, and take the top $d_t$ vectors of $\tilde{U}_t$ \end{center} \end{minipage} } (Qt);
\draw[->] (Qt) to node[above]{\begin{minipage}{1cm} \begin{center} rotate, \\ truncate \end{center}\end{minipage}} (barQt);
\draw[->] (barQt) to node[above]{\GS$([\Xprev|\bar{Q}_t])$} (Pt);
\draw[->] (Tt) to[out=300,in=180] node[below]{actual spectrum} (eigstuff2);
\draw[->] (eigstuff2) to node[below]{ \begin{minipage}{2cm} \begin{center} take the top $d_t$ singular vectors \end{center} \end{minipage} } (Qt2);
\end{tikzpicture}
\end{center}
\normalsize
\caption{Schematic of the first part of \Deflate.  The top line indicates how $W_t$ is formed from the matrix $T_t$.  We will show that $\bar{Q}_t$ approximates $U^{(t)}$, the next chunk of singular vectors in $\Nprev$, and this will imply by induction that $W_t$ approximates $U^{(\leq t)}$.  The second line in the figure  indicates some notation which will be useful for our analysis, but which is not used by the algorithm.}\label{fig:part1}
\end{figure}
We will first argue that the matrix $N_{t-1}$ is close to the truncated, subsampled, noisy estimate $T_t$. 
\begin{lemma}\label{lem:truncate1}
Let $T_t$ be as in Algorithm \ref{alg:deflate}, and choose any constant $\Cgapconstant > 0$.   
Suppose that the inductive hypotheses \eqref{h2} and \eqref{eq:s} hold.
Suppose that $p_t$ is as in the statement of Lemma \ref{lem:correlate}.  
Then, for a sufficiently large choice of $\Cerr$ in the hypothesis \eqref{h2} (depending only on $\Cgapconstant$),
with probability at least $1 - 1/n^2$,
\[ \opnorm{T_t - \Nprev} \leq \frac{\sigma_{r_{t-1}+1} + \eps\opnorm{M}}{2\Cgapconstant k^{5/2}}. \]
\end{lemma}

\begin{proof}
Write
\[ A - \Xprev \Yprev^\trans = \Nprev + \inparen{ \Mprev - \Xprev \Yprev^\trans} =: \Nprev + E_{t-1} =: \NprevEst. \]
Let $\Truncate$ denote the \textsc{Truncate} operator.
As in Algorithm \ref{alg:deflate}, consider 
\[T_t = \Truncate( P_{\Omega_t}(\NprevEst), \tau_t) = P_{\Omega_t}\inparen{ \Truncate (\NprevEst, p_t\tau_t) },\]
where as in Line \ref{line:start},
$\tau_t = \frac{ \mu^* }{np_t} \inparen{ 2k s_{t-1} + \Delta}.$
Above, use used that the sampling operation $P_{\Omega_t}$ and the truncate operator $\Truncate$ commute after adjusting for the normalization factor $p_t^{-1}$ in the definition of $P_{\Omega_t}$.
Because $\Nprev$ is incoherent, each of its entries is small.  More precisely, by the incoherence implication \eqref{eq:incNt} along with the guarantee \eqref{eq:s} on $s_{t-1}$, we have
\[ \infnorm{\Nprev} \leq \frac{\mu^*}{n} \inparen{ k\sigma_{r_{t-1}+1} + \Delta } \leq \frac{ \mu^*}{n} \inparen{ 2ks_{t-1} + \Delta } = p_t \tau_t. \]
Thus, each entry of $\NprevEst = \Nprev + E_{t-1}$ is the sum of something smaller than $p_t\tau_t$ from $\Nprev$, and an error term from $E_{t-1}$, and so truncating entrywise to $p_t\tau_t$ can only remove mass from the contribution of $E_{t-1}$.   
This implies that for all $i,j$, 
\[ \inabs{ \NprevEst -  \Truncate( \NprevEst, p_t\tau_t ) }_{i,j} \leq \inabs{ E_{t-1} }_{i,j}, \]
and so using \eqref{h2},
\begin{equation}\label{eq:truncate}
 \fronorm{ \NprevEst - \Truncate( \NprevEst, p_t\tau_t) } \leq \fronorm{E_{t-1}} \leq \sqrt{2k} \frac{ (\sigma_{r_{t-1}+1} + \eps\opnorm{M}) }{\Cerr k^{3} } =\frac{ \sqrt{2} ( \sigma_{r_{t-1}+1} + \eps\opnorm{M} ) }{\Cerr k^{5/2}}.
\end{equation}
Above, we used the fact that $E_{t-1} = M^{(<t)} - \Xprev \Yprev^\trans$ has rank at most $2k$, and hence $\fronorm{E_{t-1}} \leq \sqrt{2k}\opnorm{E_{t-1}}$.
Next, we bound the difference between $T_t$ and $\Truncate(\NprevEst, p_t\tau_t)$.
Lemma \ref{lem:matrixbernstein} in the appendix bounds the effect of
subsampling in operator norm. 
It implies that with probability $1 - 1/\poly(n)$ over the choice of
$\Omega_t$, we have
\begin{align*}
\opnorm{\Truncate(\NprevEst,p_t\tau_t) - T_t} &=
 \opnorm{ \Truncate(\NprevEst,p_t\tau_t) - P_\Omega( \Truncate(\NprevEst, p_t\tau_t)) }\\
&\lesssim
\sqrt{\frac{\max_i \twonorm{e_i^T \Truncate(\NprevEst, p_t\tau_t)}^2\log(n) }{p_t}}  + \frac{ \infnorm{\Truncate( \NprevEst, p_t\tau_t )} \log(n) }{p_t} \\
&\leq 
\sqrt{\frac{ n (p_t \tau_t)^2 \log(n) }{p_t}}  + \frac{ (p_t\tau_t) \log(n) }{p_t}\\
&\leq \inparen{ \sqrt{ \frac{\log(n)}{p_t n } } + \frac{ \log(n) }{p_t n } } \inparen{ \mu^*( 4k\sigma_{r_{t-1} + 1} + \Delta) },
\end{align*}
using the fact that
\[ p_t\tau_t = \frac{ \mu^* }{n} \inparen{ 2ks_{t-1} + \Delta } \leq \frac{\mu^*}{n} \inparen{ 4k \sigma_{r_{t-1}+1} + \Delta }\]
by \eqref{eq:s}.  Thus, our choice of $p_t$ implies that
\begin{equation}\label{eq:sparsify}
\opnorm{\Truncate(\NprevEst,p\tau_t) - T_t} \leq \eps_0  \inparen{ \sigma_{r_{t-1}+1} + \eps\opnorm{M}}.
\end{equation}
Together with \eqref{eq:truncate} we conclude that
\[ \opnorm{\Nprev - T_t} \leq 
\opnorm{ \Nprev - \NprevEst} +
\opnorm{  \NprevEst - \Truncate( \NprevEst, p\tau_t) } + 
\opnorm{ \Truncate(\NprevEst,p\tau_t) - T_t} \leq 
\eps_0 \inparen{\sigma_{r_{t-1}+1} +\eps\opnorm{M}} + \frac{ 2\sqrt{2}(\sigma_{r_{t-1}+1} + \eps\opnorm{M}) }{\Cerr k^{5/2} }  .\]
The choice of $\eps_0$ and a sufficient choice of $\Cerr$ (depending only on $\Cgapconstant$) completes the proof.

\end{proof}

Suppose for the rest of the proof that the conclusion of Lemma
\ref{lem:truncate1} holds. 
The first thing \Deflate  does after computing $T_t$ is to
obtain estimates $\widetilde{U}_t$ and $\tilde{\sigma}_1, \ldots,
\tilde{\sigma}_{k-r_t}$ for the top singular values and vectors of $T_t$.
These estimates are recovered by \SI in Line \ref{line:svd} of Algorithm
\ref{alg:deflate}.  We first wish to show that the estimated singular values
are close to the actual singular values of $T_t$.  For this, we will invoke
Theorem \ref{thm:subspaceiteration} in the appendix, which implies that as
long as the number of iterations $L$ of \SI satisfies 
\[ L \geq  Ck^{7/2}\log(n), \]
for a sufficiently large constant $C$, then
with probability $1-1/\poly(n),$ we have
\begin{equation}\label{eq:svdapprox}
 |\sigma_j(T_t) - \tilde{\sigma}_j | \leq \frac{ \opnorm{T_t} }{2\Cgapconstant k^{5/2} }  \qquad \text{for all $j$}.
\end{equation} 
Above, we took a union bound over all $j$.
Again, we condition on this event occuring. 
Thus, with our choice of $L$, the estimates $\tilde{\sigma}_j$
are indeed close to the singular values $\sigma_j(T_t)$, which by Lemma
\ref{lem:truncate1} are with high probability close to the singular values
$\sigma_{r_{t-1} + j}$ of $\Nprev$ itself.

Before we consider the next step (to $\tilde{Q}_t$) in Figure \ref{fig:part1}, consider the case when Algorithm \ref{alg:deflate} returns at line \ref{line:done}.  
Then $\tilde{\sigma}_1 \leq 10 \eps s_0 \leq 20 \eps \sigma_1$, 
and so using \eqref{eq:svdapprox} above we find that
$\opnorm{T_t} \leq 21 \eps\sigma_1.$
Then  by Lemma \ref{lem:truncate1},
\[ \sigma_{r_{t-1}+1} = \opnorm{\Nprev} \leq \opnorm{T_t} + \opnorm{\Nprev - T_t} \leq  21 \eps\sigma_1 + \frac{ \sigma_{r_{t-1}+1} + \eps\sigma_1}{2\Cgapconstant k^{5/2}} .\]
Thus, for sufficiently large $\Cgapconstant$, we conclude $\sigma_{r_{t-1}+1} \leq 22 \eps\sigma_1$.
In this case, we are done:
\begin{align*}
\opnorm{ A - \Xprev \Yprev^\trans} &\leq \opnorm{ \Mprev - \Xprev \Yprev^\trans} + \opnorm{ \Nprev }\\
&\leq \frac{ \sigma_{r_{t-1}+1} + \eps\opnorm{M} }{\Cerr k^{3} } + \sigma_{r_{t-1}+1}\\
&\leq 23\eps \sigma_1.
\end{align*}
and case (a) of the conclusion holds, as long as Lemma \ref{lem:truncate1} does.

On the other hand, suppose that Algorithm \ref{alg:deflate} does not return at line \ref{line:done} (and continue to assume that Lemma \ref{lem:truncate1} holds).  As above, \eqref{eq:svdapprox} implies that since $\widetilde{\sigma}_1 \geq 10\eps$, we must have 
\[\opnorm{T_t} \geq \frac{5\eps \sigma_1} {1 - \frac{1}{2\Cgapconstant k^{5/2} }}.\]
Then by Lemma \ref{lem:truncate1}, 
\[ \sigma_{r_{t-1} + 1} \geq  \frac{ 5\eps \sigma_1 }{1 - \frac{1}{2\Cgapconstant k^{5/2} } } - \frac{ \sigma_{r_{t-1}+1} + \eps\sigma_1 }{2\Cgapconstant k^{5/2}},\]
which implies that
\begin{equation}\label{eq:assumeeps}
\eps\sigma_1 <  \sigma_{r_{t-1} + 1}.
\end{equation}
This
establishes the conclusion \eqref{eq:assumeeps1}.
With \eqref{eq:assumeeps}, Lemma \ref{lem:truncate1} and \eqref{eq:svdapprox} together imply that
\begin{equation}\label{eq:close}
\forall r\leq n, \qquad | \sigma_r - \tilde{\sigma}_{r-r_{t-1}} | \leq \opnorm{ \Nprev - T_t } 
+ |\sigma_{r - r_{t-1}}(T_t) - \tilde{\sigma}_{r - r_{t-1}}|
\leq \frac{ \sigma_{r_{t-1} + 1}}{\Cgapconstant k^{5/2}}.
\end{equation}
Above, we use Lemma \ref{lem:weyl} in the appendix in the first inequality.

We now show that the choice of $d_t$ in Line \ref{eq:dt} of Algorithm \ref{alg:deflate}  accurately identifies a ``gap" in the spectrum.
\begin{lemma}\label{lem:gaps}
Suppose that the hypotheses and conclusions of Lemma \ref{lem:truncate1} hold, and in particular that \eqref{eq:close} holds.
Then the value $r_t = r_{t-1} + d_t$ obtained in Line \ref{eq:dt} of Algorithm \ref{alg:deflate} satisfies:
\[
	\frac{ \sigma_{r_t+1} }{\sigma_{r_t}} \leq 1 - \gamma 
\qquad \text{and} \qquad 
	\frac{ \sigma_{r_{t-1} + 1} }{ \sigma_{r_t}} \leq e.
\]
\end{lemma}
\begin{proof}
Let $d_t^*$ be the ``correct" choice of $d_t$; that is,
 $d_t^*$ be the smallest positive integer $d \leq k - r_{t-1}$ so that 
\[ 
	1 - \frac{ \sigma_{r_{t-1} + d +1 } }{ \sigma_{r_{t-1} + d } } \geq  1 - \frac{1}{k},
\]
or let $d_t^* = d - r_{t-1}$ if such an index does not exist.
Write $r_t^* = r_{t-1} + d_t^*$.
By definition, because $d_t^*$ is the smallest such $d$ (or smaller than any such $d$ in the case that $r_t^* = k$), we have 
\begin{equation}\label{eq:leqe}
 \frac{ \sigma_{r_{t-1}+1}}{\sigma_{r_t^*} } \leq \inparen{ 1 + \frac{1}{k} }^{d_t^*} \leq e. 
\end{equation}
Thus, \eqref{eq:close} reads
\begin{equation}\label{eq:weylbound}
 |\tilde{\sigma_j} - \sigma_{r_{t-1} + j} | \leq \frac{ \sigma_{r_{t-1}+1} }{\Cgapconstant k^{5/2}} \leq \frac{ e  \sigma_{r_t^*} }{\Cgapconstant k^{5/2}} . 
\end{equation}
Suppose that, for some $j \leq d_t^*$, we have
\[ \frac{ \sigma_{ r_{t-1} + j + 1 } }{ \sigma_{r_{t-1} + j} } \geq 1 - \frac{ 1}{4k}. \]
Then, using \eqref{eq:weylbound},
\begin{align*}
\frac{ \tilde{\sigma}_{j+1}}{\tilde{\sigma}_j} &\geq \frac{ \sigma_{ r_{t-1} + j + 1 } - \frac{ e \sigma_{r_t^*} }{ \Cgapconstant k^{5/2} } }{ \sigma_{r_{t-1} + j } + \frac{ e \sigma_{r_t^*} }{\Cgapconstant k^{5/2} } } 
\geq 
  \frac{ \sigma_{ r_{t-1} + j + 1 }\inparen{1  - \frac{ e }{ \Cgapconstant k^{5/2} } }}{ \sigma_{r_{t-1} + j }\inparen{ 1  + \frac{ e}{\Cgapconstant k^{5/2} }} } 
\geq 1 - \frac{1}{2k},
\end{align*}
assuming $\Cgapconstant$ is sufficiently large.  
In Algorithm \ref{alg:deflate}, we choose $d_t$ in Line \ref{eq:dt} so that there is no $j < d_t$ with 
\[ \frac{ \tilde{\sigma}_{j+1} }{\tilde{\sigma}_j} \leq 1 - \frac{1}{2k}. \]
Thus, if there were a big gap, the algorithm would have found it: more precisely, using the definition of $\gamma$, we have
\[ \frac{ \sigma_{r_t + 1} }{\sigma_{r_t}} < 1 - \frac{1}{4k} \leq 1 - \gamma. \]
This establishes the first conclusion of the lemma.
Now, a similar analysis as above shows that if for any $j \leq d_t^*$ we have
\[ \frac{ \sigma_{ r_{t-1} + j + 1 } }{ \sigma_{r_{t-1} + j} } \leq 1 - \frac{1}{k}, \]
then
\[ \frac{ \tilde{\sigma}_{j+1} }{\tilde{\sigma}_j } \leq 1 - \frac{1}{2k}, \]
assuming $\Cgapconstant$ is sufficiently large.
That is, our algorithm will always find a small gap, if it exists.  In particular, if $r_t^* < k$, we have
\[ \frac{ \sigma_{r_t^* + 1 } }{\sigma_{r_t^*} } \leq 1 - \frac{1}{k} \]
and hence $d_t \leq d_t^*$.  On the other hand, if $r_t^* = k$, then we must have $d_t = d_t^*$.
In either case,
$ d_t \leq d_t^*, $
and so
\[ \frac{ \sigma_{r_{t-1} + 1} }{\sigma_{r_t} } \leq \frac{ \sigma_{r_{t-1} + 1} }{\sigma_{r_t^*}} \leq \inparen{ 1 + \frac{1}{k} }^{d_t^*} \leq e. \]
This completes the proof of Lemma \ref{lem:gaps}.
\end{proof}
Now, we are in a position to verify the inductive hypothesis \eqref{eq:s} for the next round, in the favorable case that Lemma \ref{lem:truncate1} holds.  By definition,  we have
$s_t = \tilde{\sigma}_{d_t}$, and \eqref{eq:close}, followed by Lemma \ref{lem:gaps} implies that
\[ |\sigma_{r_t} - s_t| \leq \frac{ \sigma_{r_{t-1} + 1} }{\Cgapconstant k^2 } \leq \frac{ e\sigma_{r_t} }{\Cgapconstant k^{5/2}}. \]
In particular,
\[ \inparen{ 1 - \frac{2e}{\Cgapconstant k^{5/2}} } \sigma_{r_t} \leq s_t \leq \inparen{ 1 + \frac{2e}{\Cgapconstant k^{5/2}} } \sigma_{r_t}, \]
establishing \eqref{eq:s} for $s_t$.

Now that we know that the ``gap" structure of the singular values of $\Nprev$ is reflected by the estimates $\tilde{\sigma}_j$,
we will show that the top singular vectors are also well-approximated by the estimates $\widetilde{Q}_t$.
Recall from Algorithm \ref{alg:deflate} that $\tilde{Q}_t \in \R^{n \times d_t}$ denotes the first $d_t$ columns of $\widetilde{U}_t$, which are estimates of the top singular vectors of $T_t$.
Let $Q_t$ denote the (actual) top $d_t$ singular vectors of $T_t$.  
We will first show that $Q_t$ is close to $U^{(t)}$, and then that $Q_t$ is also close to $\widetilde{Q}_t$.
\begin{lemma}\label{lem:subspaces}
Suppose that the conclusions of Lemma \ref{lem:truncate1} and Lemma \ref{lem:gaps} hold, and that \eqref{eq:assumeeps} holds.
Then
\[ { \sin \theta( U^{(t)}, Q_t )} \leq \frac{ 4e}{ \Cgapconstant k^{3/2} }. \]
\end{lemma}
\begin{proof}
We will use a $\sin \theta$ theorem (Theorem \ref{thm:sintheta}, due to Wedin, in the appendix) to control the perturbation of the subspaces.
Theorem \ref{thm:sintheta} implies
\begin{align*}
 { \sin \theta(U^{(t)}, Q_t ) }
&\leq \frac{ \opnorm{ T_t - \Nprev } }{ \inabs{\sigma_{d_t}(T_t) - \sigma_{r_t + 1}}} 
\qquad \text{Theorem \ref{thm:sintheta}} \\
&\leq \frac{ 2e\sigma_{r_t} }{ \Cgapconstant k^{5/2} \inabs{\sigma_{d_t}(T_t) - \sigma_{r_t + 1}}}
\qquad \text{By Lemmas \ref{lem:truncate1} and \ref{lem:gaps}, and \eqref{eq:assumeeps}} \\
&\leq \frac{ 2e\sigma_{r_t}}
{
	\Cgapconstant k^{5/2} \inparen{
				\sigma_{r_t} \inparen{ 1 - \frac{2e}{\Cgapconstant k^{5/2}}} - \sigma_{r_t +1} 
			}
} 
\qquad \text{By \eqref{eq:close} and Lemma \ref{lem:gaps}} \\
&\leq \frac{ e\sigma_{r_t}}
{
	\Cgapconstant k^{5/2} \inparen{
				\sigma_{r_t} \inparen{ 1 - \frac{2e}{\Cgapconstant k^{5/2}}} - \sigma_{r_t}(1 - \gamma) 
			}
} 
\qquad \text{By Lemma \ref{lem:gaps}} \\
&\leq \frac{4e}{\Cgapconstant k^{3/2}}.
\end{align*}
\end{proof}
Now, we show that $Q_t$ is close to $\tilde{Q}_t$.
\begin{lemma}\label{lem:awkwardsubspaces}
Suppose that the conclusions of Lemma \ref{lem:truncate1} and Lemma \ref{lem:gaps} hold, and that \eqref{eq:assumeeps} holds.  Then
with probability $1 - 1/n^2$,
\[ {\sin \theta( Q_t, \tilde{Q}_t )} \leq \frac{1}{\poly(n)}. \]
\end{lemma}
\begin{proof}
By \eqref{eq:svdapprox}, Lemma \ref{lem:truncate1}, and Lemma \ref{lem:gaps}, a similar computation as in the proof of Lemma \ref{lem:subspaces} shows that
\begin{align*}
\frac{ \sigma_{d_{t}+1}(T_t) }{\sigma_{d_t}(T_t) } 
&\leq \inparen{ \frac{ \tilde{\sigma}_{d_t+1}}{\tilde{\sigma_{d_t}}  }}\inparen{1 + \frac{8e }{\Cgapconstant k^{5/2} } } \\
&\leq \inparen{ 1 - \frac{1}{4k}  }\inparen{1 + \frac{8e }{\Cgapconstant k^{5/2} } } \\
&\leq  1 - \frac{1}{k}  
\end{align*}
using the choice of $d_t$ in the second-to-last line.  Thus, by Theorem \ref{thm:subspaceiteration} in the appendix, and the 
choice of 
$L \gtrsim k \log(n)$ in \SI, we have with probability $1 - 1/\poly(n)$ that
\[  \sin\theta( Q_t, \tilde{Q}_t ) \leq \poly(n)\inparen{ 1 - \frac{1}{2k}}^L \leq \frac{1}{\poly(n)}. \]
\end{proof}
Together, Lemmas \ref{lem:subspaces} and \ref{lem:awkwardsubspaces} imply that, when Lemma \ref{lem:truncate1} and the favorable case for \SI hold,
\[ {\sin\theta( U^{(t)}, \tilde{Q}_t )} \leq \frac{ 8 e }{\Cgapconstant k^{3/2}}. \]

Finally, this implies, via Lemma \ref{lem:sinthetaconversion} in the appendix, that
there is some unitary matrix $O \in \R^{k \times k}$ so that
\[ \opnorm{ U^{(t)}O - \tilde{Q}_t } \leq \frac{16e}{\Cgapconstant k^{3/2}} , \]
and  using the fact that $U^{(t)}$ and $\tilde{Q}_t$ have rank at most $k$, we have that
\begin{equation}\label{eq:froclose}
 \fronorm{ U^{(t)}O - \tilde{Q}_t } \leq \frac{ 16\sqrt{2} e}{\Cgapconstant k }.
\end{equation}

As in Algorithm \ref{alg:deflate}, let $B$ be a random orthogonal matrix, and let $\overline{Q}_t$ be the truncation
\[ \overline{Q}_t = \text{\textsc{Truncate}}\inparen{ \tilde{Q}_t B, 8\sqrt{ \frac{\mu^* log(n)}{n}}}. \]
The reason for the random rotation is that while $U^{(t)}O$ is reasonably incoherent (because $U^{(t)}$ is), $U^{(t)}OB$ is, with high probability, even more incoherent.  More precisely, 
as in~\cite{Hardt13b}, 
we have
\begin{equation}\label{eq:rotateinc}
 \PR{ \infnorm{ U^{(t)} O B } > 8 \sqrt{ \frac{ \mu^* \log(n) }{n} }} \leq \frac{1}{n^2},
\end{equation}
where the probability is over the choice of $B$.
Suppose that the favorable case in \eqref{eq:rotateinc} occurs, so that $\infnorm{ U^{(t)} O B} \leq 8\sqrt{ \mu^* \log(n)/n}. $
In the Frobenius norm,
$\overline{Q}_t$ is the projection of $\tilde{Q}_t$ onto the (entrywise) $\ell_\infty$-ball of radius $8 \sqrt{ \mu^* log(n)/n}$ in $\R^{n\times d_t}$.
Thus,
\[ \fronorm{ \bar{Q}_t - \tilde{Q}_tB } \leq \fronorm{ X - \tilde{Q}_tB } \]
for any $X$ in this scaled $\ell_\infty$-ball, and in particular
\[ \fronorm{ \bar{Q}_t - \tilde{Q}_tB } \leq \fronorm{ U^{(t)} O B - \tilde{Q}_t B}. \]
Thus, \eqref{eq:froclose} implies that
\begin{equation}\label{eq:Qclose}
 \fronorm{ U^{(t)}OB - \overline{Q}_t } \leq \fronorm{ U^{(t)}OB - \tilde{Q}_tB } + \fronorm{ \tilde{Q}_tB -\overline{Q}_t } \leq 2\fronorm{ U^{(t)}OB -\tilde{Q}_tB} = 2\fronorm{ U^{(t)} O - \tilde{Q}_t } \leq
\frac{ 32 \sqrt{2} e }{\Cgapconstant {k} }. 
\end{equation}
Next, we consider the matrix $W_t = \text{\GS}( [X_{t-1} | \overline{Q}_t])$. 
 Because $X_{t-1}$ has orthonormal columns, this matrix has the form
$W_t = [ X_{t-1} | P_t]$, where $P_t \in \R^{n \times d_t}$ has orthonormal columns, $P_t \perp X_{t-1}$, and
\[ \mathcal{R}(P_t) = \mathcal{R}( (I - X_{t-1}X_{t-1}^\trans) \overline{Q}_t ) = \mathcal{R}(Z_t), \]
where we define
 $Z_t := (I - \Xprev\Xprev^\trans)\bar{Q}_t$ to be the projection of $\bar{Q}_t$ onto $\mathcal{R}(\Xprev)_\perp$.   Because $\bar{Q}_t$ is close to $U^{(t)}O B$, and $X_{t-1}$ is close to $U^{(<t)}$, $Z_t$ is close to $U^{(t)}O B$.  More precisely, 
\begin{align*}
\opnorm{ Z_t  - U^{(t)} O B} &\leq 
\opnorm{  (I - \Xprev\Xprev^\trans)(\overline{Q}_t -  U^{(t)}O B) } + \opnorm{ \Xprev\Xprev^\trans U^{(t)} OB } \qquad \text{ by the triangle inequality} \\
&\leq \fronorm{ \overline{Q} - U^{(t)} OB } + { \sin\theta( \Xprev, U^{(<t)} )} \\ 
&\leq \frac{ 32 \sqrt{2} e }{\Cgapconstant {k} } + \frac{1}{k^{4}} \inparen{ \frac{ \sigma_{r_{t-1}+1} + \eps \opnorm{M} }{\sigma_{r_{t-1}}} } \qquad \text{ by \eqref{eq:Qclose} and \eqref{h2} } \\
&\leq \frac{ 32 \sqrt{2} e}{\Cgapconstant {k} } + \frac{1}{k^{4}} \inparen{ \frac{ 2\sigma_{r_{t-1}+1} }{ \sigma_{r_{t-1}}} } \qquad \text{ by \eqref{eq:assumeeps} } \\
&\leq \frac{ 64 \sqrt{2} e}{\Cgapconstant {k} } \qquad \text{ for sufficiently large $k$.}
\end{align*}
Further, the Gram-Schmidt process gives a decomposition
\[ P_tR = Z_t, \]
where the triangular matrix $R$ has the same spectrum as $Z_t$.  In particular,
\[ \opnorm{R^{-1}} = \frac{1}{ \sigma_{\min}( Z_t)} \leq \frac{1}{ \opnorm{ U^{(t)} } - \frac{ 64 \sqrt{2} e }{ \Cgapconstant k} } \leq 2 \]
for sufficiently large $\Cgapconstant$.
Thus, 
\begin{align}
{ \sin \theta( U^{(\leq t)}, P_t )} &= \opnorm{ (U^{(\leq t)}_\perp)^\trans P_t } \notag\\
&= \opnorm{ (U^{(\leq t)}_\perp)^\trans Z_t R^{-1}} \notag\\
&\leq 2 \opnorm{ (U^{(\leq t)}_\perp)^\trans Z_t } \notag\\
&\leq 2 \opnorm{ (U^{(\leq t)}_\perp)^\trans U^{(t)} O B } + 2 \opnorm{ (U^{(\leq t)}_\perp)^\trans (Z_t - U^{(t)} O B) } \notag \\
&=  2 \opnorm{ (U^{(\leq t)}_\perp)^\trans (Z_t - U^{(t)} O B) } \notag \\
&\leq \frac{ 128 \sqrt{2} e }{ \Cgapconstant {k}}, \label{eq:ptclose}
\end{align}
where above we used that $(U^{(\leq t)}_\perp)^\trans U^{(t)} = 0$.
Next,
\begin{align*}
\max_i \twonorm{ e_i^\trans P_t } &\leq \max_i \twonorm{ e_i^\trans Z_t } \opnorm{R^{-1}} \\
&\leq 2\inparen{\max_i \twonorm{ e_i^\trans \bar{Q}_t } + \max_i \twonorm{ e_i^\trans \Xprev \Xprev^\trans \bar{Q}_t }} \\
&\leq 2\inparen{\max_i \twonorm{e_i^\trans \bar{Q}_t } + \max_i \twonorm{ e_i^\trans \Xprev } \inparen{\opnorm{ \Xprev^\trans U^{(t)} O B } + \opnorm{ \Xprev^\trans ( U^{(t)} O B - \bar{Q}_t ) } }}\\
&\leq 2 \inparen{\max_i \twonorm{e_i^\trans \bar{Q}_t } + \max_i \twonorm{ e_i^\trans \Xprev } \inparen{\opnorm{ \Xprev^\trans U^{(t)} } + \opnorm{  U^{(t)} O B - \bar{Q}_t  }} }\\
&\leq 16 \sqrt{  \frac{ k\mu^* \log(n) }{n}} + 2\sqrt{\frac{ k \mu_{t-1} }{n}} \inparen{ \frac{ 2}{ k^4 }  + \frac{ 32 \sqrt{2} e }{\Cgapconstant k }},
\end{align*}
where we have used the definition of $\bar{Q}_t$, the incoherence of $\Xprev$, and the computations above in the final line.
Thus,
\begin{equation}\label{eq:pflat}
 \max_i \twonorm{e_i^\trans P_t } \leq \sqrt{\frac{k}{n}} \inparen{ 16\sqrt{\mu^* \log(n)} + \frac{ \Cinc \sqrt{\mu_{t-1}} }{k}}
\end{equation}
for some constant $\Cinc$. 
Thus, when the conclusions of Lemma \ref{lem:truncate1} hold, $P_t$ is both close to $U^{(t)}$ and incoherent.  By induction, the same is true for $W_t$.  
Indeed, if $t=1$, then $P_t = W_t$, and we are done.   If $t \geq 2$, then 
we have
\[ { \sin\theta( W_t, U^{(\leq t)} ) }\leq {\sin\theta( \Xprev, U^{(\leq t-1)} )} + { \sin\theta( P_t, U^{(t)} ) }.\]
Then, the inductive hypothesis \eqref{h1} and our conclusion \eqref{eq:ptclose} imply that
\[ { \sin \theta( W_t, U^{(\leq t)} )} \leq \frac{1}{k} \]
for suitably large $\Cerr, \Cgapconstant$.
Finally, \eqref{eq:pflat}, along with the inductive hypothesis \eqref{h3} implies that
\[ \max_i \twonorm{ e_i^T W_t } \leq \max_i \twonorm{ e_i^T X_{t-1} } + \max_i \twonorm{e_i^T P_t} \leq \sqrt{\frac{k}{n}}\inparen{ \sqrt{ \mu_{t-1}}\inparen{ 1 + \frac{\Cinc}{k}}  +  16\sqrt{\mu^*\log(n)}} 
\leq \sqrt{ \frac{ k\mu_t}{n}}. \]
We remark that this last computation is the only reason we need $\sin\theta( P_t, U^{(t)} ) \lesssim 1/k$, rather than bounded by $1/4$; eventually, we will iterate and have
\[ \sqrt{\mu_T} \leq \sqrt{\mu_0}\inparen{ 1 + \frac{\Cinc}{k}}^T + 16 T \sqrt{\mu^* \log(n)} \leq e^{\Cinc} \sqrt{\mu_0} + 16 T \sqrt{ \mu^* \log(n)},\]
and we need that $\inparen{1 + \frac{\Cinc}{k} }^T \leq e^{\Cinc}$ is bounded by a constant (rather than exponential in $T$).

Finally, we have shown that
with probability $1 - 1/n^2$ (that is, in the case that Lemma \ref{lem:truncate1} holds and \SI works), all of the conclusions of Lemma \ref{lem:correlate} hold as well.
This completes the proof of Lemma \ref{lem:correlate}.

\section{Proof of Lemma \ref{lem:parameterchasing}}\label{sec:proofofmainlemma}
In the proof of Lemma \ref{lem:parameterchasing} we will need an explicit
description of the subroutine \SmoothGS that we include in Algorithm
\ref{alg:smoothGS}.
\begin{algorithm}
\KwIn{Matrix $S\in\R^{n\times k},$ parameters $\mu,\zeta >0.$}
$X\leftarrow \text{\GS}(S), H\leftarrow 0$\;
$\sigma \leftarrow \zeta\|S\|/n.$\;
\While{  $\mu(R) > \mu$ and $\sigma\le\|S\|$ }
{
$R \leftarrow \text{\GS}(S + H)$ where $H\sim \mathrm{N}(0,\sigma^2/n)$\;
 $\sigma \leftarrow 2\sigma$\;
}
\KwOut{Matrix $R$}
\caption{\SmoothGS$(S,\zeta,\mu)$ (Smooth Orthonormalization)}
\label{alg:smoothGS}
\end{algorithm}

To prove Lemma \ref{lem:parameterchasing}, we will induct on the iteration
$\ell$ in \AltLS (Algorithm \ref{alg:altmin}).
Let $R_\ell$ denote the approximation in iteration $\ell$. 
Thus, $R_0 = X_{t-1}$.
Above, we are suppressing the dependence of $R_\ell$ on the epoch number $t$, and in general, for this section we will drop the subscripts $t$ when there is no ambiguity.
We'll use the
shorthand
\[ \Theta_\ell^{j} = \theta( \xlj{R_{\ell}}, \xlj{U} )\]
and
\[ E_\ell^j = (I - \xlj{R_\ell} (\xlj{R_\ell})^\trans) \xlj{U}, \]
so that $\opnorm{E_\ell^j} = \sin (\Theta_\ell^j)$.
Recall the definition \eqref{eq:gammastar} that
$ \gamma^* = \min \inset{ \gamma, \gamma_k}.$ 
Notice that this choice ensures that
$\gamma^* \leq \gamma_{r_j}$
for all choices of $j$, including the case of $j = t$, in the final epoch of \Deflate, when $r_t = k$.

We will maintain the following inductive hypothesis:
\begin{equation}\label{i1}
 \sigma_{r_j} \tan\Theta_\ell^j \leq \max\inset{ \inparen{\frac{ 2e\sigma_{r_t}   }{k}}\exp(-\gamma^*\ell/2) , \frac{\sigma_{r_t+1} + \eps\opnorm{M}}{ 2e \Cerr k^{4}} }  =: \nu_\ell \qquad \forall j \leq t. \tag{J1}
\end{equation}
Above, the tangent of the principal angle obeys
\begin{equation}\label{eq:tangentsin}
 \opnorm{ E_{\ell}^j } \leq \tan \Theta_{\ell}^j
= \frac{ \opnorm{ E_{\ell-1}^j }}{ \sqrt{ 1 - \opnorm{ E_{\ell-1}^j }^2 }} \leq 2\opnorm{ E_{\ell-1}^j },
\end{equation}
whenever $\opnorm{E_{\ell-1}^j} \leq 1/4$.
We will also maintain the inductive hypothesis
\begin{equation}\label{inc}
\max_i \twonorm{ e_i^\trans R_{\ell} } \leq \sqrt{ \frac{ k \mu_t  }{n} }. \tag{J2}
\end{equation}

To establish the base case of \eqref{i1} for $j=t$,
we have
\[\sigma_{r_t} \sin\theta(W_t,\xlt{U}) \leq \frac{\sigma_{r_t}}{k},\]
by conclusion \eqref{eq:muWt} of Lemma~\ref{lem:correlate}, and hence by \eqref{eq:tangentsin},
\[ \sigma_{r_t} \tan\theta(W_t, \xlt{U}) \leq \frac{ 2\sigma_{r_t}}{k}. \]
If $t=1$, then $W_t = R_0$, and we are done with the base case for \eqref{i1}; if $t \geq 2$, then for $j \leq t-1$, we have
\[ \xlj{R_0} = \xlj{X_{t-1}}. \] 
Thus, for $j \leq t-1$,
 \eqref{i1} is implied by \eqref{eq:tangentsin} again, along with the fact that
\[ \sigma_{r_j} { \sin \theta( \xlj{X_{t-1}},\xlj{U} )} \leq \frac{1}{k^4}\inparen{ \sigma_{r_{t-1}+1} + \eps\opnorm{M} } \leq \frac{e\sigma_{r_t} + \eps \opnorm{M}}{k^4} \leq \frac{2e\sigma_{r_t} }{ k^4 },\]
which is the (outer) inductive hypothesis \eqref{h1}, followed by the conclusions \eqref{eq:assumeeps1} and 
\eqref{eq:gaps} from Lemma \ref{lem:correlate}.
This establishes the base case for \eqref{i1}.
The base case for \eqref{inc} follows from the conclusion \eqref{eq:muWt} of Lemma \ref{lem:correlate} directly.

Having established \eqref{i1}, \eqref{inc} for $\ell=0$, we now suppose that they hold for $\ell-1$ and consider step $\ell$.
Notice that, by running \SmoothGS with parameter $\mu = \mu_t$, we automatically ensure \eqref{inc} for the next round of induction, and 
so our next goal is to establish \eqref{i1}.  For this, we need to go deeper into the workings of \AltLS.
The analysis of \AltLS in~\cite{Hardt13b} is based on an analysis of \NSI, given in Algorithm \ref{alg:nsi}.
\begin{algorithm}
\KwIn{Number of
iterations $L\in\mathbb{N},$ symmetric matrix $A\in\R^{n\times
n}$, initial matrix $R_0 \in \R^{n \times r}$.}
\For{ $\ell = 1, \ldots, L$}
{
 $S_\ell \leftarrow A R_{\ell-1} + \tilde{G}_{\ell}$  

 $R_\ell \leftarrow \text{\GS}(S_\ell)$
}
\KwOut{Pair of matrices $(R_L,S_L)$ }
\caption{$\text{\NSI}(A,R_0,L)$ (Noisy Subspace Iteration)}
\label{alg:nsi}
\end{algorithm}
We may view \AltLS as a special case of \NSI. 
More precisely,
let $H_\ell$ be the noise matrix added from \SmoothGS in the $\ell$'th iteration of \AltLS, and define $G^{(s)}_\ell$ to be
\begin{equation}\label{eq:defGs}
 G^{(s)}_\ell = \argmin_{S \in \R^{n \times r}} \fronorm{ P_{\Omega^{(s)}_\ell}(A - R_{\ell-1}S^\trans )}^2 - AR_{\ell-1}, 
\end{equation}
and let
\[ G_\ell = \median_s( G_\ell^{(s)} ). \]
Then we may write $R_\ell$, the $\ell$'th iterate in \AltLS, as
\[ R_\ell = \text{\SmoothGS}( A R_{\ell-1} + G_\ell ) = \text{\GS}( AR_{\ell-1} + G_\ell + H_\ell ) =: \text{\GS}( AR_{\ell - 1} + \tilde{G}_\ell)\mper\]
That is, $R_\ell$ is also the $\ell$'th iterate in \NSI, when the noise matrices are $\tilde{G}_\ell = G_\ell + H_\ell$.
We will take this view going forward, and analyze \AltLS as a special case of \NSI.
We have the following theorem, which is given in~\cite[Lemma 3.4]{Hardt13b}.
\begin{theorem}\label{thm:induct}
Let $\tilde{G}_\ell = G_\ell + H_\ell$ be as above.
Let $j \leq t$ and suppose that $\opnorm{E_{\ell-1}^j} \leq \frac{1}{4}$ and that
\[ \opnorm{  \tilde{G}_\ell } \leq \frac{\sigma_{r_j} \gamma_{r_j}}{32}.\]
Then the next iterate $R_\ell$ of \NSI satisfies
\[ \tan\theta( U^{(\leq j)}, R_{\ell-1} ) \leq \max\inset{ \frac{8 \opnorm{ \tilde{G}_\ell }}{ \sigma_{r_j} \gamma_{r_j} } ,
\tan \theta(U^{(\leq j)}, R_{\ell-1} )\exp( -\gamma_{r_j}/2 )}.
\] 
\end{theorem}

To use Theorem \ref{thm:induct}, we must understand the noise matrices $\tilde{G}_\ell=  G_\ell + H_\ell$.  We begin with $G_\ell$.
\begin{lemma}[Noise term $G_\ell$ in \NSI]\label{thm:noisealtmin}
There is a constant $C$ so that the following holds.
Fix $\ell$ and suppose that \eqref{inc} holds for $\ell-1$: that is,
$\mu(R_{\ell - 1}) \leq \mu_t.$
Let $0 <\delta <1/2$, and
suppose that the samples $\Omega_t'$ for \AltLS are sampled independently with probability 
\[ p'_t \geq C L_t\Smax \frac{ k \mu_t \log(n) }{ \delta^2 n},\]
where $L_t$ is the number of iterations of \AltLS, and $\Smax \geq C\log(n)$ is the number of trials each iteration of \AltLS performs before taking a median.
Then with probability at least $1 - 1/n^5$ over the choice of $\Omega_t'$, the noise matrix $G_\ell$ satisfies
\[ \fronorm{ G_\ell} \leq \delta \inparen{ \fronorm{\Nt} + \sum_{j=1}^n \opnorm{E_{\ell-1}^j} \fronorm{M^{(j)}} } =: \omega_{\ell-1} \]
and for all $i \in [n]$, 
\[ \twonorm{ e_i^\trans G_\ell } \leq 
 \delta \inparen{ \twonorm{e_i^\trans \Nt} + \sum_{j=1}^n \opnorm{E_{\ell-1}^j} \twonorm{e_i^\trans M^{(j)}} } =: \omega_{\ell-1}^{(i)}\mper \]
\end{lemma}
The proof of Lemma \ref{thm:noisealtmin} is similar to the analysis in~\cite{Hardt13b}.  For completeness, we include the proof in Appendix \ref{app:proofGell}.
Using the inductive hypothesis \eqref{i1}, and the fact that $\fronorm{M^{(j)}} \leq \sqrt{k}\sigma_{r_j}$ ,
\[ \omega_{\ell-1} \leq \delta \inparen{ \sum_j \opnorm{ E_{\ell-1}^j}  \inparen{ \sqrt{k} \sigma_{r_j}} + \sqrt{k} \sigma_{r_t + 1} + \Delta } \leq
\delta \inparen{ t \sqrt{k} \nu_{\ell-1} + \sqrt{k} \sigma_{r_t + 1} + \Delta}. \]
We will choose
\begin{equation}\label{eq:choosedelta}
 \delta = \frac{  \gamma^* }{4e \Cerr\Cnew k^{4}} \min \inset{ \frac{1}{\sqrt{k}}, \frac{ \eps\opnorm{M}}{\Delta} }\, ,
\end{equation}
for a constant $\Cnew$ to be chosen sufficiently large.
Observe that with this choice of $\delta$, the requirement on $p_t'$ in Lemma \ref{thm:noisealtmin} is implied by the requirement on $p_t'$ in the statement in Lemma \ref{lem:parameterchasing}.
Then the choice of $\delta$ implies
\begin{equation}\label{eq:Gell}
\fronorm{G_\ell} \leq \omega_{\ell-1} \leq \frac{\gamma^*}{4e\Cerr \Cnew k^{4}} \inparen{ t  \nu_{\ell-1} + \sigma_{r_t+1} + \eps\opnorm{M} } \leq \frac{\gamma^* 4e\Cerr}{43\Cerr\Cnew} \nu_{\ell-1} \leq \frac{ \gamma^* }{ \Cnew } \nu_{\ell-1}.
\end{equation}

Now, we turn to the noise term $H_\ell$ added by \SmoothGS. 
For a matrix $G \in \R^{n \times k}$ (not necessarily orthonormal), we will define
\[\rho(A) := \frac{n}{k} \max_{i \in [n]} \twonorm{ e_i^\trans G}^2.\]
Our analysis of $H_\ell$ relies on the following lemma from \cite{Hardt13b}.
\begin{lemma}[Lemma 5.4 in \cite{Hardt13b}]\label{lem:moritzsmoothGS}
Let $\tau > 0$ and suppose that $r_t = o(n/\log(n))$.  There is an absolute constant $C$ so that the following claim holds.  Let $G \in \R^{n \times r_t}$, and let $R \in \R^{n \times r_t}$ be an orthonormal matrix, and let $\nu \in \R$ so that $\nu \geq \max{ \opnorm{G}, \opnorm{N_tR}}$.  Assume that
\[ \mu_t \geq 2 \mu(U) + \frac{C}{\tau^2} \inparen{ \frac{ \rho(G) + \mu(U) \opnorm{(U^{(\leq t)})^\trans G }^2 + \rho(N_t R) }{\nu^2} + \log(n) }\mper\]
Then, for every $\zeta \leq \tau \nu$ satisfying $\log(n/\zeta) \leq n$, we have with probability at least $1 - 1/n^4$ that the algorithm \SmoothGS$(AR + G, \zeta, \mu_t)$ terminates in $\log(n/\zeta)$ iterations, and the output $R'$ satisfies $\mu(R') \leq \mu_t$.  Further, the final noise matrix $H$ added by \SmoothGS satisfies $\opnorm{H} \leq \tau \nu$.
\end{lemma}
We will apply Lemma \ref{lem:moritzsmoothGS} to our situation. 
\begin{lemma}[Noise term $H_\ell$ in \NSI added by \SmoothGS]\label{thm:noisealtminH}
Suppose that $k = o(n/\log(n))$.  There is a constant $\Cmu$ so that the following holds.  Suppose that
 \[ \mu_t \geq \frac{ \Cmu}{ (\gamma^*)^2 } \inparen{  \mu^* \inparen{ k + \inparen{ \frac{ k^4 \fronorm{N} }{ \eps \opnorm{M} }}^2 } + \log(n) }. \]
Suppose that the favorable conclusion of Lemma \ref{thm:noisealtmin} occurs.  
Choose $\zeta = \eps s_0 k^{-5}$, as in Algorithm \ref{alg:deflate}.  
Then, with probability at least $1 - 1/n^4$ over the
randomness of \SmoothGS, 
the output $R_{\ell}$  of $\text{\SmoothGS}(A R_{\ell-1} + G_\ell, \zeta, \mu_t)$ satisfies
\[ \mu(R_{\ell}) \leq \mu_t, \]
and the number of iterations is $O(\log(n/(\eps \opnorm{M})))$.
Further, the noise matrix $H_\ell$ satisfies
\[ \opnorm{H_\ell} \leq \frac{ \gamma^* \nu_{\ell-1} }{\Cnew}. \]
\end{lemma}

\begin{proof}
We apply Lemma \ref{lem:moritzsmoothGS} with $G = G_\ell, R = R_{\ell-1}$, and $\nu = \nu_\ell$, and 
\begin{equation}\label{eq:choosetau}
\tau = \frac{ \gamma^* }{  \Cnew  }.
\end{equation}
First, we observe that the choice of $\zeta = \eps s_0 k^{-5} \leq \eps \opnorm{M} \gamma^* k^{-4} \leq \tau \nu_{\ell-1}$ indeed satisfies the requirements of Lemma \ref{lem:moritzsmoothGS}.
Next, we verify that $\max\{ \opnorm{G_\ell}, \opnorm{ N_t R_{\ell-1} } \}\leq \nu_{\ell-1}$.  Indeed, from \eqref{eq:Gell}, 
\[ \opnorm{G_\ell} \leq \omega_{\ell-1} \leq \frac{\gamma^*}{\Cnew} \nu_{\ell-1} \leq \nu_{\ell-1}. \]
Further, we have
\[\opnorm{ N_t R_{\ell-1}} \leq \sigma_{r_t} {\sin \theta( U^{(\leq t)}, R_{\ell-1})} 
\leq \nu_{\ell-1}\]
by the inductive hypothesis \eqref{i1} for $j=t$.

Next, we compute the parameters that show up in Lemma \ref{lem:moritzsmoothGS}.
From Lemma \ref{thm:noisealtmin}, we have
\[ \rho(G_\ell) \leq \frac{ n}{r_t} \max_i \inparen{ \omega_{\ell-1}^{(i)} }^2 \]
and
\[ \mu(U)\opnorm{ U^{(\leq t)} G_\ell }^2 \leq \opnorm{G_\ell}^2 \leq \mu^* \omega^2_{\ell-1}. \]
We also have
\begin{align*}
 \rho(N_t R_{\ell-1})  &= \frac{n}{r_t} \max_i \twonorm{ e_i^\trans N_t R_{\ell-1} }^2 \\
&\leq \frac{n}{r_t} \inparen{ \max_i \twonorm{ e_i^T U^{(t:k)} }\sigma_{r_t}\twonorm{ (U^{(t:k)})^\trans R_{\ell-1}} + \max_i \twonorm{e_i^\trans N}\twonorm{R_{\ell-1}} }^2 \\
&\leq \frac{n}{r_t} \inparen{ \sqrt{\frac{ k \mu(U) }{n} }\sigma_{r_t} \opnorm{ E_{\ell-1}^t } + \sqrt{ \frac{\mu_N \fronorm{N} }{n} } }^2 \\
&\leq 2\mu^*\inparen{ \frac{ k}{r_t} \sigma_{r_t}^2 \opnorm{ E_{\ell-1}^t }^2 + \frac{\fronorm{N}^2}{r_t} } \\
&\leq 2\mu^* \inparen{ \frac{k \nu_{\ell-1}^2}{r_t} + \frac{ \fronorm{N}^2 }{r_t} },
\end{align*}
where we have used the inductive hypothesis \eqref{i1} in the final line.
Then, the requirement of Lemma \ref{lem:moritzsmoothGS} on $\mu_t$ reads
\[ \mu_t \geq 2\mu^* + \frac{ C}{\tau^2 } \inparen{ \frac{ \frac{n}{r_t} \max_i \inparen{ \omega_{\ell-1}^{(i)}}^2 + \mu^* \omega^2_{\ell-1} + 2\mu^* \inparen{ \frac{k}{r_t} \nu_{\ell-1}^2 + \frac{ \fronorm{N}^2 }{r_t} } }{ \nu_{\ell-1}^2 } + \log(n) }. \]
We have, for all $i$, 
\begin{align*}
\frac{ \omega_{\ell-1}^{(i)} }{\omega_{\ell-1} } &= \frac{ \twonorm{ e_i^\trans N_t } + \sum_{j=1}^t \opnorm{ E_{\ell-1}^j } \twonorm{e_i^\trans M^{(j)}} } { \fronorm{ N_t } + \sum_{j=1}^t \opnorm{ E_{\ell-1}^j } \fronorm{M^{(j)}} } \\
&\leq \frac{ \sigma_{r_t} \sqrt{ \nicefrac{\Delta^2 \mu^*}{ n} } + \sum_{j=1}^t \opnorm{ E_{\ell-1}^j } \sigma_{r_j} \sqrt{\nicefrac{k \mu^*}{n}} }  { \fronorm{ N_t } + \sum_{j=1}^t \opnorm{ E_{\ell-1}^j } \fronorm{M^{(j)}} } \\
&\leq \frac{ \fronorm{N_t} \sqrt{ \nicefrac{\Delta^2\mu^* }{n} } + \sum_{j=1}^t \opnorm{ E_{\ell-1}^j } \fronorm{M^{(j)}} \sqrt{\nicefrac{k \mu^*}{n}} } { \fronorm{ N_t } + \sum_{j=1}^t \opnorm{ E_{\ell-1}^j } \fronorm{M^{(j)}} } \\
&= \sqrt{ \frac{ \mu^* }{n} } \inparen{ \sqrt{k} + \Delta }.
\end{align*}
We may simplify and bound the requirement on $\mu_t$ as
\begin{align*} 
 &2\mu^* + \frac{ C}{\tau^2 } \inparen{ \frac{ \frac{n}{r_t} \max_i \inparen{  \omega_{\ell-1}^{(i)}}^2 + \mu^* \omega^2_{\ell-1} + 2\mu^* \inparen{ \frac{k}{r_t} \nu_{\ell-1}^2 + \frac{ \fronorm{N}^2 }{r_t} } }{ \nu_{\ell-1}^2 } + \log(n) } \\ 
&\qquad \leq
 2\mu^* + \frac{ C}{\tau^2 } \inparen{ \frac{ \frac{n}{r_t} \max_i \inparen{  \omega_{\ell-1}^{(i)}}^2 (\gamma^*)^2 }{ \Cnew^2 \omega_{\ell-1}^2 } + \frac{ \mu^* \inparen{ \gamma^*}^2 \omega^2_{\ell-1}}{\Cnew^2 \omega_{\ell-1}^2 } +\frac{  2\mu^* \inparen{ \frac{k}{r_t}\nu_{\ell-1}^2 + \frac{ \fronorm{N}^2 }{r_t} } }{ \nu_{\ell-1}^2 } + \log(n) } \qquad \text{ using $\nu_{\ell-1} \geq \Cnew \omega_{\ell-1} / \gamma^*$, by \eqref{eq:Gell}}\\
 &\qquad \leq 2\mu^* + \frac{ C}{\tau^2 } \inparen{ \frac{ \frac{k + \Delta^2}{r_t} \mu^* (\gamma^*)^2 }{ \Cnew^2 } + \frac{ \mu^* \inparen{ \gamma^*}^2 }{\Cnew^2  } + 2\mu^* \inparen{ \frac{k}{r_t} + \frac{ \fronorm{N}^2 }{r_t \nu_{\ell-1}^2} }  + \log(n) } \qquad \text{ by the bound on $\omega_{\ell-1}^{(i)}/\omega_\ell$, above}\\
&\qquad \leq
\frac{ C' \mu^* }{ (\gamma^*)^2 }\inparen{ \frac{ k}{r_t } + \frac{ \fronorm{ N}^2}{ \nu^2_{\ell-1} }} + \frac{ \Cnew^2 \log(n) } { (\gamma^*)^2 }\text{ by the definition of $\tau$ and gathering terms } \\
&\qquad \leq
 \frac{ \Cmu \mu^*}{(\gamma^*)^2} \inparen{ \frac{k}{r_t} + \frac{ k^8 \fronorm{N}^2 }{ \eps^2 \opnorm{M}^2 } } + \frac{\Cnew^2 \log(n)}{(\gamma^*)^2} \qquad \text{ by the fact that $\nu_{\ell-1} \geq \frac{ \eps \opnorm{M}}{2e\Cerr k^4}$.}
\end{align*}
for some constant $\Cmu$, which was the requirement in the statement of the lemma.
Thus, as long as the hypotheses of the current lemma hold, 
 Lemma \ref{lem:moritzsmoothGS} implies that with probability at least $1 - 1/n^4$, 
\[ \opnorm{H_\ell} \leq \tau \nu_{\ell-1} = \frac{ \gamma^* \nu_{\ell-1} }{\Cnew}. \]
This completes the proof of Lemma \ref{thm:noisealtminH}.
\end{proof}
Thus, using the inductive hypothesis \eqref{inc}, Lemmas \ref{thm:noisealtmin} 
and \ref{thm:noisealtminH}
imply that
as long as 
the requirements on $p'_t$ and $\mu_t$ in the statements of those lemmas are satisfied (which they are, by the choices in Lemma \ref{lem:parameterchasing}), with probability at least $1 - 2/n^4$
the noise matrices $\tilde{G}_\ell$ satisfy
\[ \opnorm{ \tilde{G}_\ell } \leq  \opnorm{ G_\ell } + \opnorm{ H_\ell} \leq  \omega_{\ell-1} + \frac{ \gamma^* \nu_{\ell-1} }{\Cnew } \leq \frac{ 2\gamma^* \nu_{\ell-1} }{\Cnew}, \] 
using \eqref{eq:Gell} in the final inequality.
Now, we wish to apply Theorem \ref{thm:induct}.  
The hypothesis \eqref{i1}, along with the conclusion \eqref{eq:assumeeps1} from Lemma \ref{lem:correlate}, immediately implies that
 \[ \opnorm{ E_{\ell-1}^t } \leq \frac{1}{k} \] for all $j \leq t,$
and so in particular the first requirement of Theorem \ref{thm:induct} is satisfied.
To satisfy the second requirement of Theorem \ref{thm:induct}, we must show that
\[ \opnorm{\tilde{G}_\ell} \leq \sigma_{r_j} \gamma_{r_j} / 32, \]
for which it suffices to show that
\begin{equation}\label{eq:needforinduct}
 \frac{ 2 \gamma^* \nu_{\ell-1} }{ \Cnew } \leq \sigma_{r_j} \gamma_{r_j} / 32. 
\end{equation}
From the definition of $\nu_{\ell-1}$, and the fact that $\gamma^* \leq \gamma_{r_j}$, we see that 
\eqref{eq:needforinduct} is satisfied for a sufficiently large choice of $\Cnew$.  Then Theorem \ref{thm:induct} implies that with probability at least $ 1- 2/n^4$, for any fixed $j$, we have 
\begin{align*}
\sigma_{r_j} \tan\Theta_\ell^j &\leq \sigma_{r_j} \max\inset{ \frac{8 \opnorm{ \tilde{G}_\ell }}{ \sigma_{r_j} \gamma_{r_j} } , 
\tan\Theta_{\ell-1}^j \exp( -\gamma_{r_j}/2 )}
\\
&\leq \max\inset{ \frac{16 \nu_{\ell-1}\gamma^* }{ \Cnew \gamma_{r_j} } , 
\nu_{\ell-1}\exp( -\gamma_{r_j}/2 )} \ \ \text{by \eqref{i1} and \eqref{eq:Gell}}\\
&\leq \nu_{\ell-1} \exp(-\gamma^*/2 ) \\
&\leq \nu_\ell
\end{align*}
provided $\Cnew$ is suitably large. 
A union bound over all $j$
 establishes \eqref{i1} for the next iteration of \AltLS.
After another union bound over 
\[ L_t = \frac{C}{\gamma^*} \log \inparen{ k \cdot \frac{ \sigma_{r_t} }{ \sigma_{r_t+1} + \eps\opnorm{M} } } \]
steps of \AltLS, for some constant $C$ depending on $\Cerr$,
we conclude that with probability at least $1 - 1/n^2$, for all $j$, 
\[ \sigma_{r_j}  \sin \theta( \xlj{R_{\ell-1}}, \xlj{U} )  
\leq \sigma_{r_j}  \tan \theta( \xlj{R_{\ell-1}}, \xlj{U} )  
\leq \frac{ \sigma_{r_t + 1} + \eps\opnorm{M} } { 2e\Cerr k^{4} }. \]
To establish the second conclusion, we note that we have already conditioned on the event that \eqref{eq:Gell} holds, and so we have
\begin{align*}
\opnorm{ \Mt - \Xt\Yt^\trans } &= \opnorm{ \Pi_X \Nt + \PiXPerp \xlt{M} + X_t(AX_t - Y_t)^\trans } \\
&\leq 
 \opnorm{ \Pi_X \Nt} + \opnorm{  \PiXPerp \xlt{M} }  +\opnorm{ X_t(AX_t - Y_t) } \\
&\leq \sigma_{r_t + 1} {\sin \theta( X_t ,\xlt{U} ) } + e\sum_{j=1}^t \sigma_{r_j} { \sin\theta ( \xlj{X_t},\xlj{U}) } + \opnorm{G_L} \\
&\leq ke \frac{\sigma_{r_t + 1} + \eps\opnorm{M}}{2e\Cerr k^{4}} +    \frac{ \gamma^* }{2e\Cerr\Cnew k^{4}} (\sigma_{r_t+1} + \eps\opnorm{M})  \qquad \text{by \eqref{eq:Gell} and the definition of $\nu_L$}\\
&\leq \frac{\sigma_{r_t+1} + \eps\opnorm{M}}{\Cerr k^{3}}.
\end{align*}
Above, we used the inequality
\begin{align*}
\opnorm{ \PiXPerp\xlt{M}} = \opnorm{ \sum_{j=1}^t \PiXPerp M^{(j)} } 
\leq \sum_{j=1}^t \opnorm{ \PiXPerp M^{(j)} } 
\leq \sum_{j=1}^t &\sigma_{r_{j-1}+1} \opnorm{ \PiXPerp U^{(j)} } \\
&\leq \sum_{j=1}^t \sigma_{r_{j-1}+1} \opnorm{ \Pi_{X^{(\leq j)}_\perp} U^{(\leq j)} } 
\leq \sum_{j=1}^t e\sigma_{r_j} { \sin \theta( X^{(\leq j)}, U^{(\leq j)} )}, 
\end{align*}
using \eqref{eq:gaps} in the final inequality.
Finally, the third conclusion, that \eqref{h3} holds, follows from the definition of \SmoothGS.


\section{Simulations}
In this section, we compare the performance of \Deflate to that of other fast algorithms for matrix completion.
In particular, we investigate the performance of \Deflate compared to the Frank-Wolfe (FW) algorithm analyzed in~\cite{JaggiS10}, and also compared to the naive algorithm which simply takes the SVD of the subsampled matrix $A_\Omega$.  
All of the code that generated the results in this section can be found online at \url{http://sites.google.com/site/marywootters}.

\subsection{Performance of \Deflate compared to FW and SVD}
To compare \Deflate against FW and SVD, we generated random rank $3$, $10,000 \times 10,000$ matrices, as follows.  First, we specified a spectrum, either $(1,1,1),  (1,1,.1),$ or $(1,1,.01)$, with the aim of observing the dependence on the condition number.   Next, we chose a random $10,000 \times 3$ matrix $U$ with orthogonal columns, and let $A = U \Sigma U^\trans$, where $\Sigma \in \R^{3\times 3}$ is the diagonal matrix with the specified spectrum.  We subsampled the matrix to various levels $m$, and ran all three algorithms on the samples, to obtain a low-rank factorization $A = XY^\trans$.

We implemented \Deflate, as described in Algorithm \ref{alg:deflate}, fixing $30,000$ observations per iteration; to increase the number of measurements, we increased the parameters $L_t$ (which were the same for all $t$).  For simplicity, we used a version of \AltLS which did not implement the smoothing in \SmoothGS or the median.   We implemented the Frank-Wolfe algorithm as per the pseudocode in Algorithm \ref{alg:fw}, with accuracy parameter $\eps = 0.05$.  We remark that decreasing the accuracy parameter did improve the performance of the algorithm (at the cost of increasing the running time), but did not change its qualitative dependence on $m$, the number of observations.  We implemented SVD via subspace iteration, as in Algorithm \ref{alg:powermethod}, with $L=100$.
\begin{algorithm}
\KwIn{Observed set of indices
$\Omega\subseteq [n]\times[n]$ of an unknown, trace $1$, symmetric matrix $A\in\R^{n\times
n}$ with entries $P_\Omega(A),$ and an accuracy parameter $\eps$.} 

Initialize $Z = vv^\trans$ for a random unit vector $v \in \R^n$.

\For{ $\ell = 1$ to $1/\eps$}{
	Let $w$ be the eigenvector corresponding to the largest eigenvalue of $-\nabla f(Z)$.

\comm{ $f(Z) := \frac{1}{2} \fronorm{ A_\Omega - Z_\Omega }^2$ }

	$\alpha_\ell \gets \frac{1}{\ell}$

	$Z \gets \alpha_\ell ww^\trans + (1 - \alpha_\ell) Z$
}
\KwOut{Trace $1$ matrix $Z$ with rank at most $1/\eps$. }
\caption{ $FW(P_\Omega(A),\Omega,\eps)$: Frank-Wolfe algorithm for Matrix Completion of symmetric matrices. }
\label{alg:fw}
\end{algorithm}

The error was measured in two ways: the Frobenius error $\fronorm{ A - XY^\trans}$, and the error between the recovered subspaces, $\sin\Theta( U, X )$.   The results are shown in Figure \ref{fig:allFight}.
\begin{figure}
\begin{center}
\includegraphics[width=.45\textwidth]{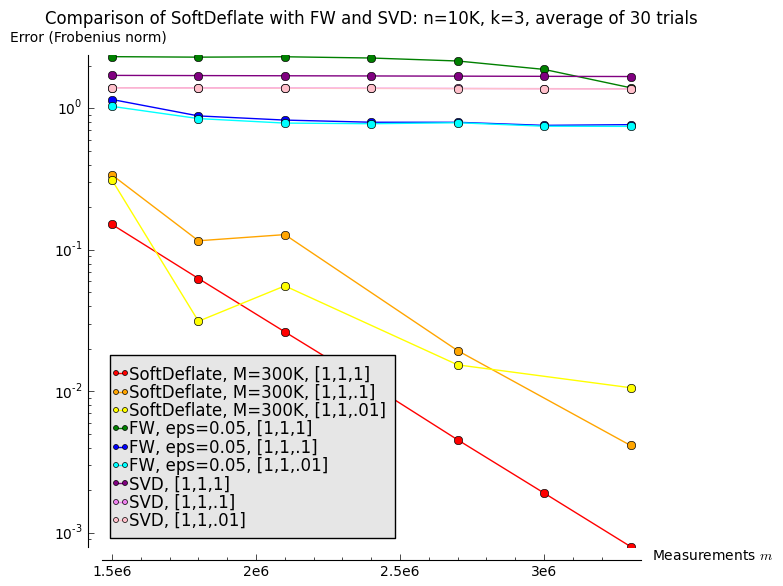}
\includegraphics[width=.45\textwidth]{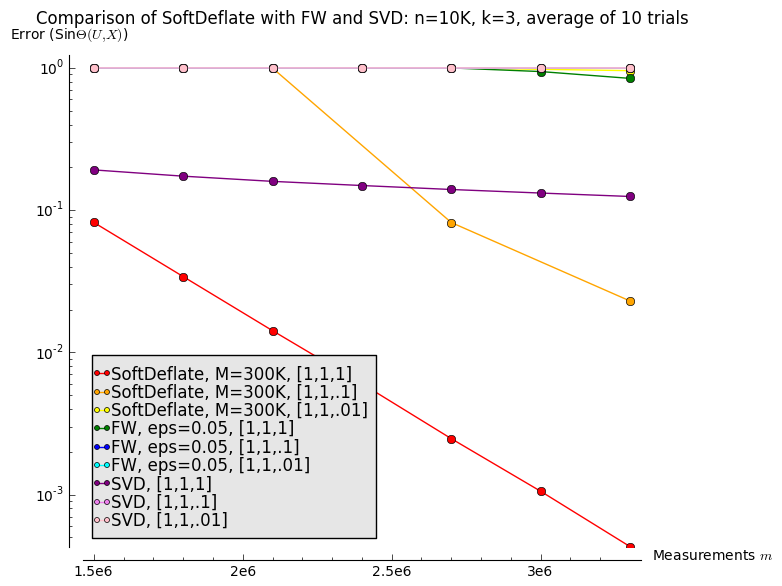}
\end{center}
\caption{Performance of \Deflate compared to FW and SVD.}
\label{fig:allFight}
\end{figure}
The experiments show that \Deflate significantly outperforms the other ``fast" algorithms in both metrics.  In particular, of the three algorithms, \Deflate is the only one which converges enough to reliably capture the singular vector associated with the $0.1$ eigenvalue;  none of the algorithms converge enough to find the $0.01$ eigenvalue with the number of measurements allowed.  The other two algorithms show basically no progress for these small values of $m$.  
To illustrate what happens when FW and SVD do converge, we repeated the same experiment for $n=1000$ and $k=2$; for this smaller value of $n$, we can let the number of measurements to get quite large compared to $n^2$.   We find that even though FW and SVD do begin to converge eventually, they are still outperformed by \Deflate.  The results of these smaller tests are shown in Figure \ref{fig:1000}.
\begin{figure}
\begin{center}
\includegraphics[width=.8\textwidth]{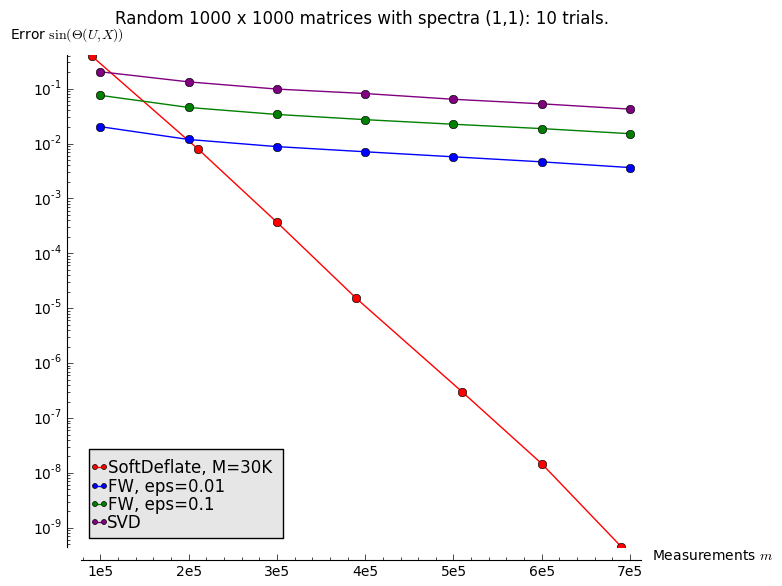}
\end{center}
\caption{Performance of \Deflate compared to FW and SVD on a rank-2, $1000\times 1000$ random matrix with spectrum $(1,1)$; average of 10 trials.}
\label{fig:1000}
\end{figure}

\subsection{Further comments on the Frank-Wolfe algorithm}
As algorithms like Frank-Wolfe are often cited as viable fast algorithms for the Matrix Completion problem, the reader may be surprised by the performance of FW depicted in Figures \ref{fig:allFight} and \ref{fig:1000}.  There are two reasons for this.  The first reason, noted in Section \ref{sec:related}, is that while FW is guaranteed to converge on the sampled entries, it may not converge so well on the actual matrix; the errors plotted above are with respect to the entire matrix.   
To illustrate this point, we include in Figure \ref{fig:fw} the results of an experiment showing the convergence of Frank-Wolfe (Algorithm \ref{alg:fw}), both on the samples and off the samples.  As above, we considered random $10,000 \times 10,000$ matrices with a pre-specified spectrum.  We fixed the number of observations at $5 \times 10^6$, and ran the Frank-Wolfe algorithm for 40 iterations, plotting its progress both on the observed entries and on the entire matrix.  
While the error on the observed entries does converge as predicted, the matrix itself does not converge so quickly.
\begin{figure}
\includegraphics[width=.8\textwidth]{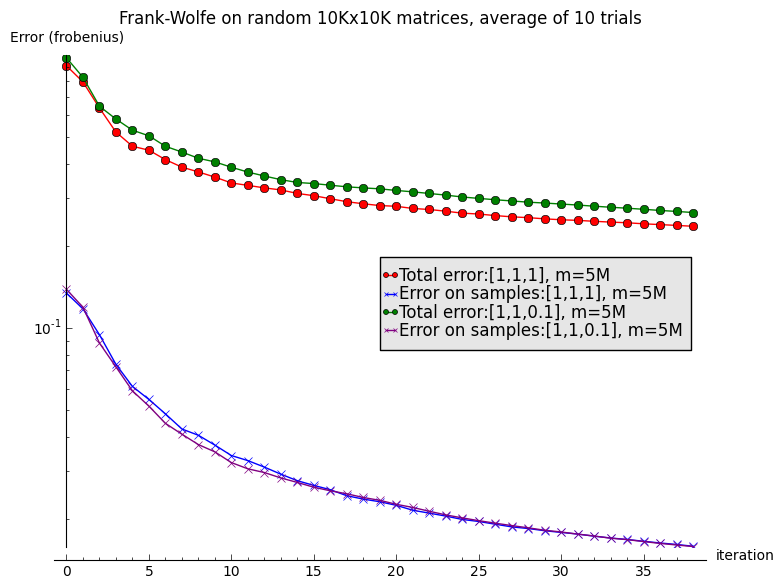}
\caption{Performance of the Frank-Wolfe algorithm on random $10,000 \times 10,000$, rank $3$ matrices with $5,000,000$ observations.  Average of $10$ trials.}
\label{fig:fw}
\end{figure}

The second reason that FW (and SVD) perform comparatively poorly above is that the convergence of FW, in the number of samples, is much worse than that of \Deflate.  More precisely, in order to achieve error on the order of $\eps$, the number of samples required by FW has a dependence of $1/\poly(\eps)$; in contrast, as we have shown, the dependence on $\eps$ of \Deflate is on the order of $\log(1/\eps)$.  In particular, because in the tests above there were never enough samples for FW to converge past the error level of $0.1$ in Figure \ref{fig:allFight}, FW never found the singular vector associated with the singular value $0.1$.  Thus, the error when measured by $\sin \Theta(U,X)$ remained very near to $1$ for the entire experiment.

\section*{Acknowledgements}
We thank the Simons Institute for Theoretical Computer Science at Berkeley, where part of this work was done.

\bibliographystyle{alpha}
\bibliography{moritz}

\appendix

\section{Dividing up $\Omega$}\label{app:subsample}
In this section, we show how to take a set $\Omega \subset [n] \times [n]$, so that each index $(i,j)$ is included in $\Omega$ with probability $p$, and return subsets $\Omega_1,\ldots,\Omega_L$ which follow a distribution more convenient for our analysis.  Algorithm \ref{alg:splitup} has the details.  Observe that the first thing that Algorithm \ref{alg:splitup} does is \em throw away \em samples from $\Omega$.  Thus, while this step is convenient for the analysis, and we include it for theoretical completeness, in practice it may be uneccessary---especially if the assumption on the distribution of $\Omega$ is an approximation to begin with.

\begin{algorithm}
\KwIn{ Parameters $p_1, \ldots, p_L$, and a set $\Omega \subset [n] \times [n]$ so that each index $(i,j)$ is included in $\Omega$ independently with probability $p = \sum_\ell p_\ell.$ }
\KwOut{ Subset $\Omega_1,\ldots, \Omega_L \subset \Omega$ so that each index $(i,j)$ is included in $\Omega_\ell$ independently with probability $p_\ell$, and so that all of the $\ell$ are independent.}

Choose 
\[ p' =  1 - \prod_{\ell=1}^L (1 - p_\ell). \]
Observe that $p' \leq p$.

Let $\Omega'$ be a set that includes each element of $\Omega$ independently with probability $p'/p$.

\Return{ \textsc{SubSample}( $p_1,\ldots, p_L$, $[n]\times[n]$, $\Omega'$ ) }

\caption{\textsc{SplitUp}: Split a set of indices $\Omega$ (as in the input to Algorithm \ref{alg:deflate}) into subsets $\Omega_1,\ldots, \Omega_t$ whose distributions are convenient for our analysis. }\label{alg:splitup}
\end{algorithm}

\begin{algorithm}
\KwIn{ Parameters $p_1,\ldots,p_L$, a universe $\cU$, and a set $\Omega \subset \cU$, so that each element $u \in \cU$ is included independently with probability $p = 1 - \prod_{\ell=1}^L (1 - p_\ell)$. }
\KwOut{ Set $\Omega_1, \ldots, \Omega_L \subset \cU$, so that each entry is included in $\Omega_\ell$ idependendently with probability $p_\ell$, and so that $\Omega_1,\ldots,\Omega_L$ are independent. }

For $r \in \{ 1, \ldots, L\}$, 
let 
\[q_r = \frac{1}{p} \sum_{S \subset \cU, |S| =r} \inparen{\prod_{\ell \in S} p_\ell }\inparen{ \prod_{\ell \not\in S} (1 - p_\ell) }. \]
Then $\sum_{r=1}^L q_r = 1$.

Initialize $L$ empty sets $\Omega_1,\ldots,\Omega_L$.

\For{$u \in \Omega$}
{
	Draw $r \in \{1, \ldots, L\}$ with probability $q_r$.

	Draw a random set $T \subset [L]$ of size $r$.

	Add $u$ to $\Omega_\ell$ for each $\ell \in T$.
}

\Return{ $\Omega_1,\ldots,\Omega_L$ }
\caption{\textsc{SubSample}: Divide a random set $\Omega$ into $L$ subsets $\Omega_1,\ldots,\Omega_L$}\label{alg:transform}
\end{algorithm}

The correctness of Algorithm \ref{alg:splitup} follows from the following lemma, about the properties of Algorithm \ref{alg:transform}.
\begin{lemma} Pick $p_1,\ldots,p_\ell \in [0,1]$, and suppose that $\Omega \subset \cU$ includes each $u \in \cU$  
independently with probability $p 1 - \prod_{\ell=1}^L (1 - p_\ell)$.  Then
the sets $\Omega_1,\ldots,\Omega_L$ returned by Algorithm \ref{alg:transform} are distributed as follows.  Each $\Omega_\ell$ is independent, and includes each $u \in \cU$ independently with probability $p_\ell$.
\end{lemma}
\begin{proof}
	Let $\cD$ denote the distribution we would like to show that that $\Omega_\ell$ follow; so we want to show that the sets returned by Algorithm \ref{alg:transform} are distributed according to $\cD$.  
\newcommand{\PRA}[1]{\mathbb{P}_{\mathcal{A}}\inbrac{#1}}
\newcommand{\PRD}[1]{\mathbb{P}_{\mathcal{D}}\inbrac{#1}}
Let $\PRA{\cdot}$ denote the probability of an event occuring in Algorithm \ref{alg:transform}, and let and $\PRD{\cdot}$ denote the probability of an event occuring under the target distribution $\cD$.  
Let $N_u$ be the random variable that counts the number of times $u$ occurs between $\Omega_1,\ldots, \Omega_\ell$.  Then observe that by definition,
\[ q_r = \PRD{ N_u = r | N_u \geq 1}, \]
and
\[ p = \PRD{ N_u \geq 1}. \]
We aim to show $\PRA{\cdot} = \PRD{\cdot}$.
First, fix $u \in \cU$, and fix any set $S \subset [L]$, and consider the event
\[ E(u,S) = \inparen{ \forall \ell \in S, u \in \Omega_\ell } \wedge \inparen{ \forall \ell \not\in S , u \not\in \Omega_\ell }. \]
We compute $\PRA{E(u,S)}$.  
\begin{align*}
\PRA{E(u,S)} &= \PRA{u \in \Omega} \sum_{r=1}^L q_r \PRA{ \text{The random set $T$ of size $r$ is precisely $S$} } \\
&= \PRD{ N_u \geq 1} \sum_{r=1}^L \PRD{ N_u = r | N_u \geq 1 } \PR{ \text{ A random subset of $[L]$ size $r$ is precisely $S$} } \\
&= \sum_{r=1}^L \PRD{N_u = r} \PR{ \text{A random subset of $[L]$ of size $r$ is precisely $S$} }\\
&= \sum_{r=1}^L \PRD{N_u = r} \PRD{ E(u, S) | N_u = r } \\
&= \PRD{ E(u,S) }.
\end{align*}
Next, we observe that for any fixed $S$, the events $\inset{ E(u,S) }_{u \in \cU}$ are independent under the distribution induced by Algorithm \ref{alg:transform}.
This follows from the fact that in all of the random steps (including the generation of $\Omega$ and within Algorithm \ref{alg:transform}), the $u \in \cU$ are treated independently.   Notice that these events are also independent under $\cD$ by definition.

Now, 
for any instantiation $\vec{\Omega'} = (\Omega_1',\ldots,\Omega_L')$ of the random variables $(\Omega_1,\ldots,\Omega_L)$, consider the event
\[ E(\vec{\Omega}') = \forall \ell, \Omega_\ell = \Omega_\ell'. \]
We have
\begin{align*}
\PRA{ E(\vec{\Omega}') } &= \PRA{ \forall u, E( u, \inset{ \ell \suchthat u\in \Omega_\ell' } ) } \\
&= \prod_{u \in \cU} \PRA{ E(u, \inset{ \ell \suchthat u \in \Omega_\ell' } ) } \quad \text{ by independence in Alg. \ref{alg:transform}} \\
&= \prod_{u \in \cU} \PRD{ E(u, \inset{ \ell \suchthat u \in \Omega_\ell' } ) } \quad \text{ by the above derivation } \\
&= \PRD{ E( \vec{\Omega}' ) } \quad \text{ by independence under $\cD$ }
\end{align*}
Thus the probability of any outcome $\vec{\Omega}'$ is the same under $\cD$ and under Algorithm \ref{alg:transform}, and this completes the proof of the lemma.
\end{proof}

\section{Useful statements}
In this appendix, we collect a few useful statements upon which we rely. 
\subsection{ Coherence bounds }
First, we record some consequences of the bound
\eqref{eq:defofmustar} on the coherence of $A$.   
We always have
\begin{equation}\label{eq:incoherenceofA}
\infnorm{A} \leq \infnorm{M} + \infnorm{N} \leq \max_{i,j}| e_i^\trans U \Lambda_U U^\trans e_j| + \infnorm{N} \leq \sigma_1 \frac{\mu^* k }{n} +  \frac{ \mu^* \fronorm{N}}{n}  \leq \frac{ \mu^* }{n} \inparen{ k\sigma_1 + \Delta },
\end{equation}
and similarly
\begin{equation}\label{eq:incoherenceofA2}
\max_i \twonorm{ e_i^\trans A } \leq \sqrt{\frac{\mu^*}{n}} \inparen{ \sqrt{k} \sigma_1 + \Delta }.
\end{equation}
It will also be useful to notice that since
$\twonorm{e_i^\trans U^{(>t)}} \leq
\twonorm{e_i^\trans U},$ \eqref{eq:defofmustar} implies that for all $t$, 
\begin{equation}\label{eq:incNt}
\infnorm{ \Nt } \leq \infnorm{M^{(>t)}} + \infnorm{N} \leq \frac{\mu^*}{n} \inparen{ k\sigma_{r_t+1} + \Delta}.
\end{equation}


\subsection{ Perturbation statements}
Next, we will use the following lemma about perturbations of singular values, due to Weyl.
\begin{lemma}\label{lem:weyl} 
Let $N, E \in \R^{n \times n}$, and let $\widetilde{N} = N + E$.  Let
$\sigma_1 \geq \sigma_2 \geq \cdots \geq \sigma_n$ denote the singular values
of $N$, and similarly let $\widetilde{\sigma_i}$ denote the singular values of
$\widetilde{N}$.  Then for all $i,$
	$\inabs{ \sigma_i - \widetilde{\sigma_i} } \leq \opnorm{E}. $
\end{lemma}

In order to compare the singular vectors of a matrix $A$ with those of a perturbed version $\tilde{A}$, we will find the following theorem helpful.
We recall that for subspaces $U, V$, $\sin\theta(U,V)$ refers to the sine of the \em principal angle \em between $U$ and $V$.  (See \cite{StewartS90} for more on principal angles).
\begin{theorem}[Thm. 4.4 in \cite{StewartS90}]\label{thm:sintheta}
Suppose that $A$ has the singular value decomposition
\[ A = \begin{bmatrix} U_1 & U_2 \end{bmatrix} \, \begin{bmatrix} \Sigma_1 & \\ &\Sigma_2 \end{bmatrix} \, \begin{bmatrix} V_1^\trans \\ V_2^\trans \end{bmatrix},\]
and let $\tilde{A} = A + E$ be a perturbed matrix with SVD
\[ A = \begin{bmatrix} \tilde{U}_1 & \tilde{U}_2 \end{bmatrix} \, \begin{bmatrix} \tilde{\Sigma}_1 & \\ &\tilde{\Sigma}_2 \end{bmatrix} \, \begin{bmatrix} \tilde{V}_1^\trans \\ \tilde{V}_2^\trans \end{bmatrix}.\]
Let
\[ R = A\tilde{V}_1 - \tilde{U}_1\tilde{\Sigma}_1 \qquad \text{and}\qquad S = A^{\trans} \tilde{U}_1 -\tilde{V}_1\tilde{\Sigma}_1.\]
Suppose there are numbers $\alpha,\delta >0 $ so that
$\sigma_{\min}(\tilde{\Sigma_1}) \geq \alpha + \delta$ and
$\sigma_{\max}(\Sigma_2) \leq \alpha.$ Then,
\[\max\inset{ \sin\Theta( U_1, V_1 ), \sin\Theta( U_2, V_2 ) } \leq 
\frac{ \max\Set{ \opnorm{R}, \opnorm{S} }}{\delta}.\]
\end{theorem}

We will also use the fact that if the angle between (the subspaces spanned by) two matrices is small, then there is some unitary transformation so that the two matrices are close.
\begin{lemma}\label{lem:sinthetaconversion}
Let $U,V \in \R^{n\times k}$ have orthonormal columns, and suppose that 
$\sin\theta(U,V) \leq \eps$
for some $\eps <1/2$.
Then there is some unitary matrix $Q \in \R^{k\times k}$ so that
$\opnorm{ UQ - V } \leq 2\eps.$
\end{lemma}
\begin{proof}
We have 
$V = \Pi_U V + \Pi_{U_\perp} V = U(U^\trans V) + \Pi_{U_\perp} V.$
Since $\sin \theta(U,V) \leq \eps$, we have
$\opnorm{\Pi_{U_\perp} V} \leq \eps,$
and
$\sigma_k( U^T V) = \cos \theta(U,V) \geq \sqrt{ 1 - \eps^2 }.$
Thus, we can write $U^T V = Q + E,$
where
$\opnorm{E} \leq 1 - \sqrt{1 -\eps^2}.$
The claim follows from the triangle inequality.
\end{proof}

\subsection{Subspace Iteration}
\sectionlabel{subsit}
Our algorithm uses the following standard version of the well-known Subspace
Iteration algorithm---also known as Power Method.

\begin{algorithm}[h]
\caption{\SI$(A, k, L)$ (Subspace Iteration)}
\label{alg:powermethod}
\KwIn{ Matrix $A$, target rank $k$, number of iterations $L$ }
$S_0 \in \R^{n \times k} \gets $ random matrix with orthogonal rows\;
\For{ $\ell = 1,\ldots, L$}
{
	$R_\ell \gets A S_{\ell-1}$\;
	$S_\ell \gets \text{\GS}(R_\ell)$\;
}
\For{ $i=1, \ldots, k$}
{
	$\tilde{\sigma}_i^2 \gets (R_L)_i^\trans A^\trans A (R_L)_i$
\comm{$(R_L)_i$ is the $i$-th column of $R_L$}
}
\Return{ $(R_L, \vec{\tilde{\sigma}} )$\;}
\KwOut{ A matrix $R \in \R^{n \times k}$ approximating the top $k$ singular vectors of $A$, and estimates $\tilde{\sigma}_1,\ldots, \tilde{\sigma}_k$ of the singular values.}
\end{algorithm}

We have the following theorem about the convergence of \SI. 
\begin{theorem}\label{thm:subspaceiteration}
Let $A \in \R^{n \times n}$ be any matrix, with singular values $\sigma_1 \geq \sigma_2 \geq \cdots \geq \sigma_n$.
Let $R_L \in \R^{n \times k}$ be the matrix with orthonormal columns returned after $L$ iterations of \SI (Algorithm \ref{alg:powermethod}) with target rank $k$.
for some suitably small parameter $\gamma < 1$.
Then the values $\tilde{\sigma}_i = (R_i)^\trans A R_i$ satisfy
\[ | \tilde{\sigma}_i - \sigma_i | \leq \sigma_i \inparen{ 1 - (1 - \gamma)^k} + 2 n \sigma_1 \inparen{ 1 - \gamma}^L. \]
In particular, if $\gamma = o(1/k)$ and if $L = C \log(n)/\gamma$
then with probability $1-1/\poly(n),$
\[ |\tilde{\sigma}_i - \sigma_i | \lesssim \frac{\sigma_1}{n} + \sigma_i k \gamma \lesssim \sigma_1 k \gamma. \]
\end{theorem}
\begin{proof}
Let $r_1 \leq r_2 \leq \cdots \leq r_t$ be the indices $r \leq k$ so that
$\sigma_{r+1}/\sigma_r \leq 1 - \gamma.$
Notice that we may assume without loss of generality that $r_t = k$.  Indeed, the result of running \SI with target rank $k$ is the same as the result of running \SI with a larger rank and restricting to the first $k$ columns of $R_\ell$.
Write
$A = \sum_j U^{(j)} \Sigma_j V^{(j)},$
where $\Sigma_j$ contains the singular values $\sigma_{r_j + 1}, \ldots, \sigma_{r_{j+1}}$.  
Then using~\cite[Chapter 6, Thm 1.1]{Stewart01} 
and deviation bounds for the
principal angle between a random subspace and fixed subspace, we have
\[ 
\Pr\Set{\sin \theta\inparen{ U^{(j)}, R_L^{(j)} } \leq C n^c\inparen{1 -
\gamma}^L}\ge 1-1/n^{c'}\mper
\]
Here, $c'$ can be made any constant by increasing $c$ and $C$ is an absolute
constant. Fix $i$ and let $x_i = (R_L)_i$ denote the $i-th$ column of $R_L$.  
Suppose that $i \in \inset{ r_j + 1, \ldots, r_{j+1} }$.
Then, the estimates $\tilde{\sigma_i}$ satisfy
\begin{align*}
\tilde{\sigma}_i &= x_i^\trans A^\trans A x_i
= \twonorm{ A^{(j)} x_i } + \sum_{s \neq j} \twonorm{ A^{(s)} x_i }^2.
\end{align*}
The second term satisfies
\[ \sum_{s \neq j} \twonorm{ A^{(s)} x_i }^2 \leq  \sigma_1^2 \sin^2\theta( U^{(s)}, R_L^{(s)} ) \leq \sigma^2_1 n^2 (1 - \gamma)^{2L}.\]
The first term has
\[ \twonorm{ A^{(j)} x_i }^2 \leq \opnorm{ A^{(j)} }^2 = \sigma^2_{r_{j+1}} \]
and
\[ \twonorm{ A^{(j)} x_i }^2 \geq \cos^2\theta\inparen{ U^{(s)}, R_L^{(s)} } \cdot \sigma_{\min}(A^{(j)}) \geq \inparen{1 - n^2(1 - \gamma)^{2L}} \cdot \sigma^2_{r_j}. \]
By definition, as there are no significant gaps between $\sigma_{r_j+1}$ and $\sigma_{r_j}$, we have
\[ \frac{ \sigma_{r_j + 1} }{\sigma_{r_{j+1}}} \geq (1 - \gamma)^k, \]
and so this completes the proof after collecting terms.
\end{proof}

\subsection{Matrix concentration inequalities}
We will repeatedly use the Matrix Bernstein and Matrix Chernoff inequalities; we use the versions from~\cite{Tropp12}:

\begin{lemma}\label{lem:mbernstein}[Matrix Bernstein~\cite{Tropp12}]
Consider a finite sequence $\inset{ Z_k }$ of independent, random, $d\times d$  matrices.  Assume that each matrix satisfies 
\[ \EE X_k = 0 , \qquad \| X_k \| \leq R \text{ almost surely.} \]
Define
\[ \sigma^2 := \max \inset{ \norm{ \sum_k \EE X_k X_k^\trans } , \norm{ \sum_k \EE X_k^\trans X_k } }\mper\]
Then, for all $t \geq 0$, 
\[ \PR{ \norm{ \sum_k X_k } \geq t } \leq 2d\exp\inparen{ \frac{ -t^2/2 }{\sigma^2 + R/3 } }. \]
\end{lemma}

One corollary of Lemma \ref{lem:mbernstein} is the following lemma about the concentration of the matrix $P_{\Omega}(A)$.
\begin{lemma}\label{lem:matrixbernstein}
Suppose that $A \in \R^{n\times n}$ and let $\Omega \subset [n] \times [n]$ be a random subset where each entry is included independently with probability $p$.  Then
\[
\PR{ \opnorm{ P_{\Omega}(A) - A} > u } \leq 
2n \exp\inparen{ 
	\frac{-u^2/2}{\inparen{\frac{1}{p} - 1} \inparen{ \max_i \twonorm{e_i^\trans A}^2 + \frac{u}{3}  \infnorm{A} }}
}.
\]
\end{lemma}
\begin{proof}
Let $\xi_{ij}$ be independent Bernoulli-$p$ random variables, which are $1$ if $(i,j) \in \Omega$ and $0$ otherwise.
\[ P_\Omega(A) - A = \sum_{i,j} \inparen{\frac{\xi_{ij}}{p} - 1} A_{i,j} e_ie_j^T,\]
which is a sum of independent random matrices.  Using the Matrix Bernstein
inequality, Lemma \ref{lem:mbernstein}, we conclude that
\[\PR{ \opnorm{ P_{\Omega}(A) - A} > u } \leq 2n \exp\inparen{ \frac{-u^2/2 }{ \sigma^2 + Ru/3 }},\]
where
\[ \sigma^2 = \opnorm{\EE \sum_{i,j} \inparen{\frac{\xi_{ij}}{p} - 1}^2 A_{i,j}^2 e_ie_j^Te_je_i^T} = \inparen{\frac{1}{p} -1 } \max_i \twonorm{A_i}^2 \]
and
\[  \opnorm{\inparen{\frac{\xi_{ij}}{p} - 1} A_{i,j} e_ie_j^T} \leq R = \inparen{ \frac{1}{p} -1}\infnorm{A} \]
almost surely.
This concludes the proof.
\end{proof}

Finally, we will use the Matrix Chernoff inequality.
\begin{lemma}\label{lem:mchernoff}[Matrix Chernoff~\cite{Tropp12}]
Consider a finite sequence $\inset{ X_k }$ of independent, self-adjoint, $d\times d$ matrices.  Assume that each $X_k$ satisfies
\[ X_k \succcurlyeq 0, \qquad \lambda_{\max}( X_k ) \leq R \text{ \ \ almost surely. } \]
Define
\[ \mu_{\min} := \lambda_{\min}\inparen{ \norm{ \sum_k \EE X_k } }, \qquad \mu_{\max} := \lambda_{\max}\inparen{ \norm{ \sum_k \EE X_k } }. \]
Then for $\delta \in (0,1)$, 
\[ \PR{ \lambda_{\min}\inparen{\sum_k X_k } \leq (1 - \delta) \mu_{\min}} \leq d\exp( - \delta^2 \mu_{\min} /2R ) \]
and
\[ \PR{ \lambda_{\max} \inparen{ \sum_k X_k } \geq (1 + \delta) \mu_{\max} } \leq d \exp( -\delta^2 \mu_{\max} / 3R ). \]
\end{lemma}


\subsection{ Medians of vectors }
For $v \in \R^k$, let $\median(v)$ be the entry-wise median.
\begin{lemma}\label{lem:vectormedians}
Suppose that $v^{(s)}$, for $s=1,\ldots,T$ are i.i.d. random vectors, so that for all $s$,
\[ \PR{ \twonorm{v^{(s)}}^2 > \alpha } \leq 1/5. \]
Then
\[ \PR{ \twonorm{ \median( v^{(s)} ) }^2 > 4\alpha } \leq \exp( -\Omega(T) ). \]
\end{lemma}
\begin{proof}
Let $S \subset [T]$ be the set of $s$ so that $\twonorm{v^{(s)}}^2 \leq \alpha.$  By a Chernoff bound,
\[ 
\PR{ |S| \leq \frac{ 3T}{4 } } = 
\PR{ \sum_{s=1}^T \ind{ \twonorm{ v^{(s)} }^2 > \alpha } > \frac{ T}{4} } \leq \exp( -\Omega(T) ). \] 
Suppose that the likely event occurs, so $|S| > 3T/4$.  For $j \in [k]$, 
let 
\[ S_j = \inset{ s \in S \suchthat (v_j^{(s)})^2 \geq \median_s(( v_j^{(s)})^2) }. \]
Because $|S| > 3T/4$, we have $|S_j| \geq T/4$.  Then
\begin{align*}
\twonorm{ \median_s( v_j^{(s)} ) }^2 &= \sum_{j=1}^n \median_s\inparen{ ( v_j^{(s)})^2 } 
\leq \sum_{j=1}^n \frac{1}{|S_j|} \sum_{s \in S_j} (v_j^{(s)})^2 \\
&\leq \sum_{j=1}^n \frac{4}{T} \sum_{s \in S_j} (v_j^{(s)})^2 
\leq \sum_{j=1}^n \frac{4}{T} \sum_{s \in S} (v_j^{(s)})^2 
\leq \frac{4}{T} \sum_{s \in S} \twonorm{ v^{(s)} }^2 
\leq \frac{ 4|S|\alpha }{T} 
\leq 4\alpha.
\end{align*}
This completes the proof. 
\end{proof}

\section{Proof of Lemma \ref{thm:noisealtmin}}\label{app:proofGell}
In this section, we prove Lemma \ref{thm:noisealtmin}, which bounds the noise matrices $G_\ell^{(s)}$.
The proof of Lemma \ref{thm:noisealtmin} is similar to the analysis in \cite{Hardt13b}, Lemmas 4.2 and 4.3. 
For completeness, we  include the details here.
Following Remark \ref{rem:subsample}, we assume that sets $\Omega_\ell^{(s)}$ are independent random sets, which include each index independently with probability  
\[ p' := \frac{ p_t' }{ \Smax L_t }. \]
Consider each noise matrix $G_\ell^{(s)}$, as in \eqref{eq:defGs}.  
In Lemma 4.2 in \cite{Hardt13b}, an explicit expression for $G^{(s)}_\ell$ is derived:
\begin{proposition}\label{prop:Gdecomp}
Let $G_\ell^{(s)}$ be as in \eqref{eq:defGs}.  Then we have
\[ G^{(s)}_\ell = (G^{(s)}_\ell)^M + (G^{(s)}_\ell)^N, \]
where
\[ e_i^\trans ( G^{(s)}_\ell)^M = e_i^\trans M_t (I - R_{\ell-1} R_{\ell-1}^\trans) P^{(s)}_i R_{\ell-1} (B^{(s)}_i)^{-1}.
\]
and
\[ e_i^\trans (G^{(s)}_\ell)^N = e_i^\trans \inparen{ N_tP^{(s)}_i R_{\ell-1} (B^{(s)}_i)^{-1} - N_tR_{\ell-1} }. \]
Above,
$P_i^{(s)}$ is the projection onto the coordinates $j$ so that $(i,j) \in \Omega_\ell^{(s)}$, and
\[ B^{(s)}_i = R_{\ell-1}^\trans P^{(s)}_i R_{\ell -1}. \]
\end{proposition}
We first bound the expression for $(G^{(s)}_\ell)^M$ in terms of the decomposition in Proposition \ref{prop:Gdecomp}.
Let 
\[ D_{\ell-1}^j = (I - R_{\ell-1}R_{\ell-1}^\trans )U^{(j)}.\]
Thus, $D_{\ell-1}^j$ is similar to $E_{\ell-1}^j$, and more precisely we have
\begin{equation}\label{eq:DE}
\norm{D_{\ell-1}^j} \leq \norm{E_{\ell-1}^j}. 
\end{equation}
To see \eqref{eq:DE}, observe that (dropping the $\ell$ subscripts for readability)
\begin{align*}
\opnorm{E^j} &= \max_{ \twonorm{x} = 1, \twonorm{y} = 1} x^\trans E^j y \\
&= \max_{x,y} x^\trans \left[ \begin{array}{c|c} (R^{(j+1:t)})^\trans U^{(<j)} & (R^{(j+1:t)})^\trans U^{(j)} \\\hline
(R_\perp)^\trans U^{(<j)} & (R_\perp)^\trans U^{(j)} \end{array}\right] y \\
&\geq \max_{ x = (0,x'), y = (0,y') } (x')^\trans (R_\perp)^\trans U^{(j)} y' \\
&= \opnorm{D^j}
\end{align*}

First, we observe that with very high probability, $B_i^{(s)}$ is close to the identity.
\begin{claim}\label{claim:Bi}
There is a constant $C$ so that the following holds.
Suppose that $p' \geq C k\mu_t \log(n) / (n\delta^2)$.  Then
\[ \PR{ \lambda_{\min} (B_i^{(s)}) \leq 1 - \delta/2 
\text{\ \ or \ \ }
 \lambda_{\max} (B_i^{(s)}) \geq 1 + \delta/2 } \leq 1/n^5. \] 
\end{claim}
\begin{proof}
We write
\[ B_i^{(s)} = R_{\ell-1}^T P_i^{(s)} R_{\ell-1} = \sum_{r=1}^n \frac{1}{p'} \xi_r ( R_{\ell-1}^\trans e_r) (e_r^\trans R_{\ell-1}), \]
where $\xi_r$ is $1$ with probability $p'$ and $0$ otherwise.
We
apply the Matrix Chernoff bound (Lemma \ref{lem:mchernoff}); we have
\[ \opnorm{\frac{1}{p'} \xi_r ( R_{\ell-1}^\trans e_r) ( e_r^\trans R_{\ell-1}) } \leq \frac{ \twonorm{ e_r^\trans R_{\ell-1} }^2 }{ p'} \leq \frac{ \mu_t k }{np' } \quad \text{almost surely,} \]
and
$\lambda_{\min}( \EE B_i^{(s)} ) = \lambda_{\max}( \EE B_i^{(s)} ) = 1$.  Then Lemma \ref{lem:mchernoff} implies that
\[ \PR{ \lambda_{\min}(B_i^{(s)}) \leq 1 - \delta/2 \text{   or   } \lambda_{\max}(B_i^{(s)}) \geq 1 + \delta/2 }
\leq n \exp( -\delta^2 p'n / (8\mu_t k ) ) + n \exp( - \delta^2 p'n / (12 \mu_t k) ). \]
The claim follows from the choice of $p'$.
\end{proof}
Next, we will bound the other part of the expression for $(G_\ell^{(s)})^M$ in Proposition \ref{prop:Gdecomp}.
\begin{claim}\label{claim:maybespiky} There is a constant $C$ so that the following holds. Suppose that $p' \geq \frac{ C \mu_t k }{ n\delta^2}$.  Then 
for each $s$, 
\[ \PR{ \twonorm{ e_i^T M_t ( I - R_{\ell-1}R_{\ell-1})^\trans P^{(s)}_i R_{\ell-1} } \geq \frac{\delta}{4} \inparen{ \sum_{j=1}^t \twonorm{ e_i^\trans M^{(j)} } \opnorm{ E_{\ell-1}^j } } } \leq \frac{1}{20}. \]
\end{claim}
\begin{proof} We compute the expectation of $\twonorm{ e_i^T M_t ( I - R_{\ell-1}R_{\ell-1})^\trans P_i^{(s)} R_{\ell-1} } $ and use Markov's inequality.
For the proof of this claim, let $Y = M_i(I - R_{\ell-1}R_{\ell-1}^\trans )$.
\begin{align*}
\EE \twonorm{ e_i^T Y P_i^{(s)} R_{\ell-1} }^2 &= \EE e_i^T Y P_i^{(s)} R_{\ell-1} R_{\ell-1}^\trans P_i^{(s)} Y^\trans e_i \\
&= e_i^\trans Y \EE\inparen{ P_i^{(s)} R_{\ell-1} R_{\ell-1}^\trans P_i^{(s)} } Y^\trans e_i \\
&= e_i^\trans Y \inparen{ R_{\ell-1}R_{\ell-1}^\trans + \inparen{ \frac{1}{p'} - 1 } \diag_r\left( \twonorm{e_r^T R_{\ell-1}}^2 \right) } Y^\trans e_i \\
&= \twonorm{ e_i^\trans Y R_{\ell-1}}^2 + \inparen{ \frac{1}{p'} - 1 } \sum_{r=1}^n \twonorm{ e_r^\trans R_{\ell-1} }^2 (Y_{i,r})^2 \\
&= \inparen{ \frac{1}{p'} - 1 } \sum_{r=1}^n \twonorm{ e_r^\trans R_{\ell-1} }^2 (Y_{i,r})^2 \\
&\leq \twonorm{ e_i^\trans Y }^2 \inparen{ \frac{1}{p'} - 1 } \inparen{ \frac{ \mu_t  k }{n}  } \\
&\leq \frac{\delta^2\twonorm{ e_i^\trans Y}^2}{400},
\end{align*}
using the fact that $Y R_{\ell-1} = 0$, and finally our choice of $p'$ (with an appropriately large constant $C$).
Now, using \eqref{eq:DE},
\[ \twonorm{ e_i^\trans Y} = \twonorm{ e_i^\trans U^{(\leq t)} \Lambda_{(t)} U^{(\leq t)} (I - R_{\ell-1}R_{\ell-1}^\trans)} \leq \sum_{j=1}^t \twonorm{ e_i^\trans M^{(j)}} \opnorm{ D_{\ell-1}^j } \leq \sum_{j=1}^t \twonorm{ e_i^\trans M^{(j)} } \opnorm{ E_{\ell-1}^j}. \]
Along with Markov's inequality, this completes the proof.
\end{proof}
Finally, we control the term $(G_\ell^{(s)})^N$.
\begin{claim}\label{claim:Gn}
There is a constant $C$ so that the following holds.
Suppose that $p' \geq C k\log(n) \mu_t /(\delta^2 n)$ for a constant $C$.  Then for each $s \leq T$,
\[
\PR{ \twonorm{ e_i^\trans (G_\ell^{(s)})^N } \geq \frac{\delta}{4} \twonorm{ e_i^\trans N_t } } \leq \frac{1}{15}. \]
\end{claim}
\begin{proof}
Using Proposition \ref{prop:Gdecomp}, 
\begin{align*}
e_i^\trans (G_\ell^{(s)})^N &= e_i^\trans \inparen{ \Nt P_i^{(s)} R_{\ell-1} \inparen{ B_i^{(s)}}^{-1} - \Nt R_{\ell-1} } \\
&= e_i^\trans \inparen{ \Nt P_i^{(s)} R_{\ell-1}  - \Nt R_{\ell-1}B_i^{(s)} }(B_i^{(s)})^{-1} \\
&= \inparen{ e_i^\trans \Nt (P_i^{(s)} - I )R_{\ell-1} + e_i^\trans \Nt R_{\ell-1}(I - B_i^{(s)} ) } \inparen{ B_i^{(s)}}^{-1} \\
&=: (y_1 + y_2)\inparen{ B_i^{(s)}}^{-1}.
\end{align*}
We have already bounded $\opnorm{ (B_i^{(s)})^{-1}}$ with high probability in Claim \ref{claim:Bi}, when the bound on $p'$ holds, and so we now bound $\twonorm{y_1}$ and $\twonorm{y_2}$ with decent probability.
As we did in Claim \ref{claim:maybespiky}, we compute the expectation of $\twonorm{y_1}^2$ and use Markov's inequality.
\begin{align*}
\EE \twonorm{ y_1}^2 &= \EE \twonorm{ e_i^\trans \Nt \inparen{ P_i^{(s)} - I } R_{\ell-1} }^2 \\
&= e_i^\trans \Nt \EE \inbrak{ (P_i^{(s)} - I) R_{\ell-1}R_{\ell-1}^\trans (P_i^{(s)} - I ) } \Nt^\trans e_i \\
&= e_i^\trans \Nt \inparen{ \frac{1}{p'} - 1} \diag_r( \twonorm{ e_r^\trans R_{\ell-1} }^2 ) \Nt^\trans e_i \\
&= \inparen{ \frac{1}{p'} - 1} \sum_{r=1}^n \inparen{ \Nt }_{ir}^2 \twonorm{ e_r^\trans R_{\ell-1} }^2 \\
&\leq \inparen{ \frac{ \mu_t k }{np'} } \twonorm{ e_i^\trans \Nt }^2. 
\end{align*}
Thus, by Markov's inequality, we have
\[ \PR{ \twonorm{ y_1} \geq 20 \sqrt{ \frac{ \mu_t k }{np'} } \twonorm{ e_i^\trans \Nt }} \leq \frac{1}{20}. \]
Next, we turn our attention to the second term $\twonorm{y_2}$.  We have
\[ \twonorm{y_2} = \twonorm{ e_i^\trans \inparen{ \Nt  R_{\ell-1} } \inparen{ I - B_i^{(s)}}  } \leq \twonorm{ e_i^\trans \Nt R_{\ell-1} } \opnorm{ I - B_i^{(s)}}. \]
By Claim \ref{claim:Bi}, we established that with probability $1 - 1/n^5$, $\opnorm{ I - (B_i^{(s)}) } \leq \frac{\delta}{2}$, with our choice of $p'$.  
Thus, with probability at least $1 - 1/n^5$, 
\[ \twonorm{ y_2 } \leq \delta \twonorm{ e_i^\trans \Nt R_{\ell-1} } \leq \frac{\delta}{2} \twonorm{ e_i^\trans \Nt }. \]
Altogether, we conclude that with probability at least $1 - 1/20 - 2/n^5$, we have
\[ \twonorm{ e_i^\trans (G_\ell^{(s)})^N } \leq \inparen{ \twonorm{y_1} + \twonorm{y_2} }\opnorm{ (B_i^{(s)})^{-1}}
\leq  \frac{ 3\delta }{4(1 - \delta/2)} \twonorm{ e_i^\trans \Nt } \leq \delta \twonorm{ e_i^\trans \Nt }\]
as long as $\delta \leq 1/2$.  This proves the claim.
\end{proof}

Putting Claims \ref{claim:Bi}, \ref{claim:maybespiky} and \ref{claim:Gn} together, along with the choice of $p_t' = L_t\Smax p'$, we conclude that, for each $s \in [T]$ and for any $\delta < 1/2$,
\begin{equation}\label{eq:Gssmall} 
\PR{ \twonorm{e_i^\trans G^{(s)}_\ell } \geq \frac{ \delta}{4(1 - \delta/2)} \inparen{ \twonorm{ e_i^\trans \Nt} + \sum_{j=1}^t \twonorm{ e_i^\trans M^{(j)} } \opnorm{ E_{\ell-1}^{j} } } } \leq \frac{ 1}{5}. 
\end{equation}
This implies that
\[ \twonorm{e_i^\trans G_\ell } = \twonorm{ e_i^\trans{ \median_s G_\ell^{(s)} } } = \twonorm{ \median_s ( e_i^\trans G_\ell^{(s)} ) } \]
is small with exponentially large probability.  Indeed, by Lemma \ref{lem:vectormedians}, 
\begin{equation*} 
\PR{ \twonorm{e_i^\trans G_\ell } \geq \frac{ \delta}{2(1 - \delta/2)} \inparen{ \twonorm{ e_i^\trans \Nt} + \sum_{j=1}^t \twonorm{ e_i^\trans M^{(j)} } \opnorm{ E_{\ell-1}^{j} } } }\leq \exp(-c \Smax),
\end{equation*}
for some constant $c$.
By the choice of $\Smax$, the failure probability is at most $1/n^6$, and a union bound over all $i$ shows that, with probability at least $ 1- 1/n^5$, 
\begin{equation}\label{eq:Gsmall}
 \twonorm{ e_i^\trans G_\ell } \leq \delta \inparen{  \twonorm{ e_i^\trans \Nt} + \sum_{j=1}^t \twonorm{ e_i^\trans M^{(j)} } \opnorm{ E_{\ell-1}^{j} } } = \omega_{\ell-1}^{(i)}. 
\end{equation}
This was the second claim in Lemma \ref{thm:noisealtmin}.  
Now, we show that in the favorable case that \eqref{eq:Gsmall} holds, so does the first claim of Lemma \ref{thm:noisealtmin}, and this will complete the proof of the lemma.  Suppose that \eqref{eq:Gsmall} holds.  Then 
\begin{align*}
\fronorm{ G_\ell } &= \sqrt{ \sum_{i=1}^n \twonorm{e_i^\trans G_\ell }^2} \\
&\leq \sqrt{ \sum_{i=1}^n \delta^2 \inparen{  \twonorm{ e_i^\trans \Nt} + \sum_{j=1}^t \twonorm{ e_i^\trans M^{(j)} } \opnorm{ E_{\ell-1}^{j} } }^2} \\
&\leq 
 \delta \sqrt{ \sum_{i=1}^n  \twonorm{ e_i^\trans \Nt}^2 } + \delta \sqrt{ \sum_{i=1}^n \inparen{ \sum_{j=1}^t \twonorm{ e_i^\trans M^{(j)} } \opnorm{ E_{\ell-1}^{j}} }^2 } \\
&\leq 
 \delta \fronorm{\Nt} + \delta \sqrt{ \sum_{i=1}^n \inparen{ \sum_{j=1}^t \twonorm{ e_i^\trans M^{(j)} } \opnorm{ E_{\ell-1}^{j}} }^2 } \mper
\end{align*}
Notice that, for any real numbers $(a_{i,j})$, $i \in [n], j \in [t]$, and for any real number $b_j$, $j \in [t]$, we have
\begin{align*}
\inparen{ \sum_{i=1}^n \inparen{ \sum_{j=1}^t a_{i,j} b_j }^2 }^{1/2} &= \twonorm{Ab} 
= \max_{ \twonorm{z} = 1} z^\trans A b 
= \max_{ \twonorm{z} = 1} \sum_{j=1}^t (z^\trans A e_j ) b_j \\
&\leq \sum_{j=1}^t \max_{z^{(j)}} ( (z^{(j)})^\trans A e_j ) b_j 
= \sum_{j=1}^t \twonorm{Ae_j} b_j 
= \sum_{j=1}^t \inparen{ \sum_{i=1}^n a_{i,j}^2 }^{1/2} b_j.
\end{align*}
Thus, we may bound the second term above by
\begin{align*}
\delta \sqrt{ \sum_{i=1}^n \inparen{ \sum_{j=1}^t\twonorm{ e_i^\trans M^{(j)} } \opnorm{ E_{\ell-1}^j } }^2 } 
&\leq \delta \sum_{j=1}^t \inparen{ \sum_{i=1}^n \twonorm{ e_i^\trans M^{(j)} }^2 }^{1/2} \opnorm{ E_{\ell-1}^j } \notag\\
&= \delta \sum_{j=1}^t \fronorm{ M^{(j)} } \opnorm{ E_{\ell - 1}^j }. 
\end{align*}
Altogether, we conclude that, in the favorable case the \eqref{eq:Gsmall} holds, 
\[ \fronorm{G_\ell} \leq \delta\inparen{ \fronorm{\Nt} + \sum_{j=1}^t \fronorm{M^{(j)}} \opnorm{ E_{\ell-1}^j }} = \omega_{\ell-1}, \]
as desired.
This completes the proof of Lemma \ref{thm:noisealtmin}.

\end{document}